\documentclass[letterpaper]{article} 
\usepackage{aaai23}  
\usepackage{times}  
\usepackage{helvet}  
\usepackage{courier}  
\usepackage[hyphens]{url}  
\usepackage{natbib}  
\usepackage{caption} 
\makeatother
\usepackage{algorithmic}
\usepackage{subcaption}
\usepackage{apxproof}

\providecommand{\customgenericname}{}
\newcommand{\newcustomtheorem}[2]{%
  \newenvironment{#1}[1]
  {%
  \renewcommand\customgenericname{#2}%
  \renewcommand\theinnercustomgeneric{##1}%
  \innercustomgeneric
  }
  {\endinnercustomgeneric}
}
\newcustomtheorem{customtheorem}{Theorem}

\usepackage{microtype}
\usepackage{multirow}
\usepackage[pdftex]{graphicx}
\usepackage{xcolor}
\usepackage{subcaption}
\usepackage{booktabs} 
\usepackage[utf8]{inputenc} 
\usepackage[T1]{fontenc}    
\usepackage{url}            
\usepackage{caption}
\usepackage{amsfonts}       
\usepackage{nicefrac}       
\usepackage{microtype}      
\usepackage{bm}
\usepackage{diagbox}
\usepackage{color} 
\usepackage{amssymb}
\usepackage{amsmath}
\usepackage{amsfonts}
\usepackage{mathtools}
\usepackage[toc,page,header]{appendix}
\usepackage{minitoc}

\usepackage{mathrsfs}
\usepackage{amsthm}
\usepackage{amsmath}
\usepackage{amssymb, bm}
\usepackage{comment}
\usepackage{multirow}
\usepackage{tabularx}
\usepackage{multicol}
\usepackage[normalem]{ulem}
\usepackage{thmtools} 
\usepackage{thm-restate}
\usepackage{graphicx}
\usepackage{caption}

\usepackage{wrapfig}

\usepackage{multicol, multirow}

\newcommand\blfootnote[1]{%
\begingroup
\renewcommand\thefootnote{}\footnote{#1}%
\addtocounter{footnote}{-1}%
\endgroup
}
\usepackage[linesnumbered,ruled,vlined]{algorithm2e}

\newtheorem{theorem}{Theorem}
\newtheorem{definition}{Definition}
\newtheorem{proposition}{Proposition}
\newtheorem{corollary}{Corollary}
\newtheorem{lemma}{Lemma}

\newcommand{\cP}{{\mathcal P}}

\newcommand{\cS}{{\mathcal S}} 

\newcommand{\cM}{{\mathcal M}}

\SetKwInput{KwInput}{Input}
\SetKwInput{KwOutput}{Output}
\SetKwComment{Comment}{/* }{ */}

\newtheorem{innercustomthm}{Theorem}
\newenvironment{customthm}[1]
  {\renewcommand\theinnercustomthm{#1}\innercustomthm}
  {\endinnercustomthm}
\providecommand{\customgenericname}{}

\usepackage{cleveref}

\usepackage{amsfonts}
\usepackage{amsmath}
\usepackage{amssymb}
\usepackage{amsthm}
\usepackage{bm}
\usepackage{mathrsfs}
\usepackage{mathtools}
\usepackage{thm-restate}
\usepackage{thmtools}

\usepackage[linesnumbered,ruled,vlined]{algorithm2e}
\SetKwInput{KwInput}{Input}
\SetKwInput{KwOutput}{Output}

\usepackage{color}

\title{Learning to Shape Rewards using a Game of Two
Partners}
\author{David Mguni$^\dag$, Taher Jafferjee, Jianhong Wang,
Nicolas Perez-Nieves,\\ Tianpei Yang, Matthew Taylor,
Wenbin Song,
\\Feifei Tong,
Hui Chen,
Jiangcheng Zhu, Jun Wang, 
Yaodong Yang$^\dag$ 
}
\affiliations{

}

\begin{document}

\maketitle

\begin{abstract}
    Reward shaping (RS)\blfootnote{$^\dag$Corresponding authors <david.mguni@hotmail.com>\\<yaoodong.yang@pku.edu.cn>.} is a powerful method in reinforcement learning (RL) for overcoming the problem of sparse or uninformative rewards. However, RS typically relies on manually engineered shaping-reward functions whose construction is time-consuming and error-prone. It also requires domain knowledge which runs contrary to the goal of autonomous learning. We introduce Reinforcement Learning Optimising Shaping Algorithm (ROSA), an automated reward shaping framework in which the shaping-reward function is constructed in a Markov game between two agents. A reward-shaping agent ({\fontfamily{cmss}\selectfont Shaper}) uses \textit{switching controls} to determine which states to add shaping rewards 
    for more efficient learning while the other agent ({\fontfamily{cmss}\selectfont Controller})  learns the optimal policy for the task using these shaped rewards. We prove that ROSA, which adopts existing RL algorithms, learns to construct a shaping-reward function that is beneficial to the task thus ensuring efficient convergence to high performance policies. 
    We demonstrate ROSA's properties in three didactic experiments and show its superior performance against state-of-the-art RS algorithms in challenging sparse reward environments. 
    
\end{abstract}










\section{Introduction}\label{Section:Introduction}


\looseness=-2
Despite the notable success of RL in a variety domains, 
enabling  RL algorithms to learn successfully in numerous real-world tasks remains a challenge \cite{wang2021multi}.  
A key obstacle to the success of RL algorithms is that sparse reward signals can hinder agent learning \cite{charlesworth2020plangan}. In many settings of interest such as physical tasks and video games, 
rich informative signals of the agent's performance are not readily available \cite{hosu2016playing}. For example, in 
the video game Super Mario \cite{shao2019survey}, the agent must perform sequences of hundreds of actions while receiving no rewards for it to successfully complete its task. 
In this setting, the infrequent feedback of the agent's performance leads to RL algorithms requiring large numbers of samples (and high expense) for solving problems \cite{hosu2016playing}. Therefore, there is  need for RL techniques to solve such problems efficiently. 
%
%


Reward shaping (RS) is a tool to introduce additional rewards, known as \textit{shaping rewards}, to supplement the environmental reward. These rewards can encourage exploration and insert structural knowledge in the absence of informative environment rewards thereby significantly improving learning outcomes \cite{devlin2011empirical}. RS algorithms often assume hand-crafted and domain-specific shaping functions, constructed by subject matter experts, which runs contrary to the aim of autonomous learning. Moreover, poor choices of shaping rewards can \textit{worsen} the agent's performance \cite{devlin2011theoretical}. To resolve these issues, a useful shaping reward must be obtained \emph{autonomously}. 
\looseness=-10

We develop a framework that autonomously constructs shaping rewards during learning. 
%
ROSA introduces an additional RL agent, the {\fontfamily{cmss}\selectfont Shaper}, that \textit{adaptively} learns to construct shaping rewards by observing {\fontfamily{cmss}\selectfont Controller} 
, while  {\fontfamily{cmss}\selectfont Controller} learns to solve its task. This \emph{generates tailored shaping rewards without the need for domain knowledge or manual engineering}. These shaping rewards successfully promote effective learning, addressing this key challenge. 


The resulting framework is a two-player, nonzero-sum Markov game (MG) \cite{shoham2008multiagent} --- an extension of Markov decision process (MDP) that involves \textit{two} independent learners with distinct objectives.
In our framework the two agents have distinct learning agendas but cooperate to achieve the {\fontfamily{cmss}\selectfont Controller}'s objective. An integral component of ROSA is a novel combination of RL and \textit{switching controls} \cite{mguni2022timing, bayraktar2010one,mguni2018viscosity}. This enables {\fontfamily{cmss}\selectfont Shaper}  to quickly determine useful states to learn to add in shaping rewards (i.e., states where adding shaping rewards improve the {\fontfamily{cmss}\selectfont Controller}'s performance) but  disregard other states. In contrast {\fontfamily{cmss}\selectfont Controller} must learn actions for every state.  This 
leads to the {\fontfamily{cmss}\selectfont Shaper} quickly finding shaping rewards that guide the {\fontfamily{cmss}\selectfont Controller}'s learning process toward optimal trajectories (and away from suboptimal trajectories, as in Experiment 1).

This approach tackles multiple obstacles.
First, a new agent ({\fontfamily{cmss}\selectfont Shaper}) learns while the {\fontfamily{cmss}\selectfont Controller} is training, while avoiding convergence issues. Second, unlike standard RL, the {\fontfamily{cmss}\selectfont Shaper}'s learning process uses \textit{switching controls}. We show successful empirical results and also prove ROSA converges.
\section{Related Work }

\looseness=-2

\subsubsection*{Reward Shaping} 
Reward Shaping (RS) adds a \textit{shaping function} $F$ to supplement the agent's reward to boost learning. 
%
RS however has some critical limitations. First, RS does not offer a means of finding $F$. 
Second, poor choices of $F$ can \textit{worsen} the agent's performance \cite{devlin2011theoretical}. Last, adding shaping rewards can change the underlying problem therefore generating policies that are completely irrelevant to the task \cite{mannion2017policy}. In \cite{ng1999policy} it was established that potential-based reward shaping (PBRS) which adds a shaping function of the form $F(s_{t+1},s_{t})=\gamma\phi(s_{t+1})-\phi(s_t)$ preserves the optimal policy of the problem. Recent variants of PBRS include potential-based advice which defines $F$ over the state-action space \cite{harutyunyan2015expressing} and approaches that include time-varying shaping functions \cite{devlin2012dynamic}. Although the last issue can be addressed using potential-based reward shaping (PBRS) \cite{ng1999policy}, the first two issues remain \cite{behboudian2021policy}. 
%
%

To avoid manual engineering of $F$, useful shaping rewards must be obtained autonomously. Towards this \cite{zou2019reward} introduce an RS method that adds a shaping-reward function prior which fits a distribution from data obtained over many tasks. Recently, \cite{hu2020learning} use a bilevel technique to learn a scalar coefficient for an already-given shaping-reward function. 
Nevertheless, constructing $F$ \textit{while training} can produce convergence issues since the reward function now changes during training \cite{igl2020impact}. Moreover, while $F$ is being learned the reward can be corrupted by inappropriate signals that hinder learning. 

%
%
%
\subsubsection*{Curiosity Based Reward Shaping} Curiosity Based Reward Shaping aims to encourage the agent to explore states by rewarding the agent for novel state visitations using exploration heuristics. One approach is to use state visitation counts \cite{ostrovski2017count}. More elaborate approaches such as \cite{burda2018exploration} introduce a measure of \textit{state novelty} using the prediction error of features of the visited states from a random network. Pathak et al. \citeyear{pathak2017curiosity} use the prediction error of the next state from a learned dynamics model and Houthooft et al. \citeyear{houthooft2016vime} maximise the information gain about the agent’s belief of the system dynamics.
In general, these methods provide no performance guarantees nor do they ensure the optimal policy of the underlying MDP is preserved. Moreover, they naively reward exploration without consideration of the environment reward. This can lead to spurious objectives being maximised (see Experiment 3 in Sec. \ref{Section:Experiments}). 


\looseness -10
Within these two categories, closest to our work are bilevel approaches for learning the shaping function \cite{hu2020learning,stadie2020learning}. Unlike Hu et al.\citeyear{hu2020learning} whose method requires a useful shaping reward to begin with, ROSA constructs a shaping reward function from scratch leading to a fully autonomous method. Moreover, in Hu et al. \citeyear{hu2020learning} and Stadie et al.\citeyear{stadie2020learning}, the agent's policy and shaping rewards are learned with \textit{consecutive} updates. In contrast, ROSA performs these operations \textit{concurrently} leading to a faster, more efficient procedure.  Also in contrast to Hu et al. \citeyear{hu2020learning} and Stadie et al. \citeyear{stadie2020learning}, ROSA learns shaping rewards only at relevant states, this confers high computational efficiency (see Experiment 2, Sec. \ref{Section:Experiments})).  As we describe, ROSA, which successfully \textit{learns the shaping-reward function} $F$, uses a similar form as PBRS. However in ROSA, $F$ is augmented to include the actions of another RL agent to learn the shaping rewards online. Lastly, unlike curiosity-based methods e.g., \cite{burda2018exploration} and \cite{pathak2017curiosity}, our method preserves the agent's optimal policy for the task (see Experiment 3, Sec. \ref{Section:Experiments}) and introduces intrinsic rewards that promote complex learning behaviour (see Experiment 1, Sec. \ref{Section:Experiments}) .
\section{Preliminaries \& Notations}
The RL problem is typically formalised as a Markov decision process $\left\langle \mathcal{S},\mathcal{A},P,R,\gamma\right\rangle$ where $\mathcal{S}$ is the set of states, $\mathcal{A}$ is the discrete set of actions, $P:\mathcal{S} \times \mathcal{A} \times \mathcal{S} \rightarrow [0, 1]$ is a transition probability function describing the system's dynamics, $R: \mathcal{S} \times \mathcal{A} \rightarrow \mathbb{R}$ is the reward function measuring the agent's performance, and  $\gamma \in (0, 1]$ specifies the degree to which the agent's rewards are discounted \cite{sutton2018reinforcement}. At time $t$ the system is in state $s_{t} \in \mathcal{S}$ and the agent must
choose an action $a_{t} \in \mathcal{A}$ which transitions the system to a new state 
$s_{t+1} \sim P(\cdot|s_{t}, a_{t})$ and produces a reward $R(s_t, a_t)$. A policy $\pi:\mathcal{S} \times \mathcal{A} \rightarrow [0,1]$ is a distribution over state-action pairs where $\pi(a|s)$ is the probability of selecting action $a\in\mathcal{A}$ in state $s\in\mathcal{S}$. The agent's goal is to
find a policy $\pi^\star\in\Pi$ that maximises its expected returns given by: 
$
v^{\pi}(s)=\mathbb{E}[\sum_{t=0}^\infty \gamma^tR(s_t,a_t)|a_t\sim\pi(\cdot|s_t)]$ where $\Pi$ is the agent's policy set. We denote this MDP by $\mathfrak{M}$.

\looseness=-2


\subsubsection*{Two-player Markov games}
A two-player Markov game (MG) is an augmented MDP involving two agent that simultaneously take actions over many rounds \cite{shoham2008multiagent}. In the classical MG framework, each agent's rewards and the system dynamics are now influenced by the actions of \textit{both} agents. Therefore, each agent $i\in\{1,2\}$ has its reward function $R_i:\mathcal{S}\times\mathcal{A}_1\times\mathcal{A}_2\to\mathbb{R}$ and action set $\mathcal{A}_i$ and its goal is to maximise its \textit{own} expected returns. The system dynamics, now influenced by both agents, are described by a transition probability $P:\mathcal{S} \times\mathcal{A}_1\times\mathcal{A}_2 \times \mathcal{S} \rightarrow [0, 1]$. As we discuss in the next section, ROSA induces a specific MG in which the dynamics are influenced by \textit{only} {\fontfamily{cmss}\selectfont Controller}.
$
%
%
%
%
%
$
%
%


\subsubsection*{Reward shaping}
Reward shaping (RS) seeks to promote more efficient learning by inserting a (state dependent) shaping reward function $F$. Denote  by $\tilde{v}$ the objective function that contains a shaping reward function $F$ and by $\tilde{\pi}\in \tilde{\Pi}$ the corresponding policy i.e., $v^{\tilde{\pi}}(s)=\mathbb{E}[\sum_{t=0}^\infty \gamma^t(R(s_t,a_t)+F(\cdot))|a_t\sim\tilde{\pi}(\cdot|s_t)]$. Let us also denote by $v^{\pi_k}$  the expected return after $k$ learning steps, the goal for RS can be stated as inserting a shaping reward function $F$ for any state $s\in\mathcal{S}$: 

\textbf{C.1.}\hspace{5 mm} $v^{\tilde{\pi}_m}(s)\geq v^{\pi_m}(s)$ for any $m\geq N$,

\textbf{C.2.}\hspace{5 mm} $\underset{\pi\in\Pi}{\arg\max}\;    \tilde{v}^{\pi}(s)\equiv\underset{\pi\in\Pi}{\arg\max}\;    v^\pi(s)$,
 
where $N$ is some finite integer.

 Condition \textbf{C.1} ensures that RS produces a performance improvement (weakly) during the learning process i.e., RS induces more efficient learning and does not degrade performance (note that both value functions measure the expected return from the environment only). Lastly, Condition \textbf{C.2} ensures that RS preserves the optimal policy.\footnote{For sufficiently complex tasks, a key aim of an RS function is to enable the agent to acquire rewards more quickly provided the agent must learn an improvement on its initial policy that is to say $v^{\tilde{\pi}_n}(s)> v^{\pi_n}(s)$  for all $n\geq N$; whenever $\underset{\pi\in\Pi}{\max}\;    v^\pi(s)> v^{\pi_0}(s)$. However such a condition cannot be guaranteed for all RL tasks since it is easy to construct a trivial example in which RS is not required and the condition would not hold.}

Poor choices of $F$ \textit{hinder} learning \cite{devlin2011theoretical} in violation of (ii), and therefore RS methods generally rely on hand-crafted shaping-reward functions that are constructed using domain knowledge (whenever available). 
    In the absence of a useful shaping-reward function $F$, the challenge is to \textit{learn} a shaping-reward function that leads to more efficient learning while preserving the optimal policy. 
    The problem therefore can be stated as finding a function $F$ such that (i) - (iii) hold.
Determining this function is a significant challenge; poor choices can hinder the learning process, moreover attempting to learn the shaping-function while learning the RL agent's policy presents convergence issues given the two concurrent learning processes \cite{zinkevich2006cyclic}. Another issue is that using a hyperparameter optimisation procedure to find $F$ directly does not make use of information generated by intermediate  state-action-reward tuples of the RL problem which can help to guide the optimisation. 

\section{Our Framework}
We now describe the problem setting, details of our framework, and how it learns the shaping-reward function. We then describe {\fontfamily{cmss}\selectfont Controller}'s and {\fontfamily{cmss}\selectfont Shaper}'s objectives. We also describe the switching control mechanism used by  the {\fontfamily{cmss}\selectfont Shaper} and the learning process for both agents.   

The {\fontfamily{cmss}\selectfont Shaper}'s goal is to construct shaping rewards
to guide the {\fontfamily{cmss}\selectfont Controller} towards quickly learning $\pi^\star$. 
%
%
%
%
%
%
To do this, the {\fontfamily{cmss}\selectfont Shaper} learns how to choose the values of a shaping-reward at each state. 
Simultaneously, {\fontfamily{cmss}\selectfont Controller} performs actions to maximise its rewards using its own policy. Crucially, the two agents tackle distinct but complementary set problems. The problem for  {\fontfamily{cmss}\selectfont Controller} is to learn to solve the task by finding its optimal policy, the problem for the {\fontfamily{cmss}\selectfont Shaper} is to learn how to add shaping rewards to aid {\fontfamily{cmss}\selectfont Controller}.  The objective for the {\fontfamily{cmss}\selectfont Controller} is given by: 
\begin{align}
\tilde{v}^{\pi,\pi^2}(s)=\mathbb{E}\left[\sum_{t=0}^\infty \gamma^t\left(R(s_t,a_t) +\hat{F}(a^2_t,a^2_{t-1})\right)\Big|s=s_0\right], \nonumber
\end{align}
where $a_t\sim \pi$ is the {\fontfamily{cmss}\selectfont Controller}'s action, $\hat{F}$ is the shaping-reward function which is given by $
    \hat{F}(a^2_t,a^2_{t-1})\equiv a^2_t-\gamma^{-1}a^2_{t-1}$, $a^2_t:$ is chosen by  the {\fontfamily{cmss}\selectfont Shaper} (and $a^2_t\equiv 0, \forall t<0$) using the policy $\pi^2:\mathcal{S}\times\mathcal{A}_2\to[0,1]$ where $\mathcal{A}_2$ is the action set for the {\fontfamily{cmss}\selectfont Shaper}. 
%
Note that the shaping reward is state dependent since the {\fontfamily{cmss}\selectfont Shaper}'s policy is contingent on the state. The 
set
$\mathcal{A}_2$ is a subset of $\mathbb{R}^p$ and can therefore be for example a set of integers $\{1,\ldots,K\}$ for some $K\geq 1$.
With this, the {\fontfamily{cmss}\selectfont Shaper} constructs a shaping-reward based on the agent's environment interaction, therefore the shaping reward is tailored for the specific setting. The transition probability $P:\mathcal{S}\times\mathcal{A}\times\mathcal{S}\to[0,1]$ takes the state and \textit{only} the {\fontfamily{cmss}\selectfont Controller}'s actions as inputs.
Formally, the MG is defined by a tuple $\mathcal{G}=\langle \mathcal{N},\mathcal{S},\mathcal{A},\mathcal{A}_2,P,\hat{R}_1,\hat{R}_2,\gamma\rangle$ where the new elements are $\mathcal{N}=\{1,2\}$ which is the set of agents, $\hat{R}_1:=R+\hat{F}$ is the new {\fontfamily{cmss}\selectfont Controller} reward function which now contains a shaping reward $\hat{F}$, the function $\hat{R}_2:\mathcal{S}\times\mathcal{A}\times\mathcal{A}_2\to\mathbb{R}$ is the one-step reward for the {\fontfamily{cmss}\selectfont Shaper} (we give the details of this function later).

As the {\fontfamily{cmss}\selectfont Controller}'s policy can be learned using any RL method, ROSA easily adopts any existing RL algorithm for the {\fontfamily{cmss}\selectfont Controller}.  Note that unlike reward-shaping methods e.g. \cite{ng1999policy}, our shaping reward function $F$ consists of actions $a^2$ which are chosen by  the {\fontfamily{cmss}\selectfont Shaper} which enables a shaping-reward function to be learned online.  We later prove an policy invariance result (Prop. \ref{invariance_prop}) analogous to that in \cite{ng1999policy} and show ROSA preserves the optimal policy of the agent's underlying MDP.

\subsection{Switching Controls} 
So far the {\fontfamily{cmss}\selectfont Shaper}'s problem involves learning to construct shaping rewards at \textit{every} state including those that are irrelevant for guiding {\fontfamily{cmss}\selectfont Controller}. To increase the (computational) efficiency of the {\fontfamily{cmss}\selectfont Shaper}'s learning process, we now introduce a form of policies known as \textit{switching controls}. Switching controls enable {\fontfamily{cmss}\selectfont Shaper} to decide at which states to learn the value of shaping rewards it would like to add. Therefore, now {\fontfamily{cmss}\selectfont Shaper} is tasked with learning how to shape {\fontfamily{cmss}\selectfont Controller}'s rewards \emph{only} at states that are important for guiding {\fontfamily{cmss}\selectfont Controller} to its optimal policy. This enables {\fontfamily{cmss}\selectfont Shaper} to quickly determine its optimal policy\footnote{i.e., a policy that maximises its own objective.} $\pi^2$ for only the relevant states unlike {\fontfamily{cmss}\selectfont Controller} whose policy must learned for all states. Now at each state {\fontfamily{cmss}\selectfont Shaper} first makes a \textit{binary decision} to decide to \textit{switch on} its shaping reward $F$ for the {\fontfamily{cmss}\selectfont Controller}.
This leads to an MG in which, unlike classical MGs, the {\fontfamily{cmss}\selectfont Shaper} now uses \textit{switching controls} to perform its actions. 

We now describe how at each state both the decision to activate a shaping reward and their magnitudes are determined. Recall that $a^2_t\sim\pi^2$ determines the shaping reward through $F$. At any $s_t$, the decision to turn on the shaping reward function $F$ is decided by a (categorical) policy $\mathfrak{g}_2:\mathcal{S} \to \{0,1\}$. Therefore, $\mathfrak{g}_2$ determines whether the {\fontfamily{cmss}\selectfont Shaper} policy $\pi^2$ should be used to introduce a shaping reward $F(a^2_t,a^2_{t-1}), a^2_t\sim\pi^2$. 
We denote by $\{\tau_k\}$ the times that a switch takes place, for example, if the switch is first turned on at state $s_5$ then turned off at $s_7$, then $\tau_1=5$ and $\tau_2=7$. Recalling the role of $\mathfrak{g}_2$, the switching times obey the expression $\tau_k=\inf\{t>\tau_{k-1}|s_t\in\mathcal{S},\mathfrak{g}_2(s_t)=1\}$ and are therefore
\textit{ \textbf{rules} that depend on the state.}. The termination times $\{\tau_{2k-1}\}$ occur according to some external (probabilistic) rule i.e., if at state $s_t$ the shaping reward is active, then the shaping reward terminates at state $s_{t+1}$ with probability $p\in ]0,1]$.  Hence, by learning an optimal $\mathfrak{g}_2$, the {\fontfamily{cmss}\selectfont Shaper} learns the best states to activate $F$. 



%
We now describe the new {\fontfamily{cmss}\selectfont Controller} objective. To describe the presence of shaping rewards at times $\{\tau_{2k}\}_{k>0}$ for notational convenience, we introduce a switch $I_t$ for the shaping rewards which takes values $0$ or $1$ and obeys $I_{\tau_{k+1}}=1-I_{\tau_{k}}$ (note that the indices are the times $\{\tau_k\}$ not the time steps $t=0,1,\ldots$) and $I_t\equiv 0, \forall t\leq 0$. With this, the new {\fontfamily{cmss}\selectfont Controller} objective is: 
\begin{align}
\tilde{v}^{\pi,\pi^2}(s_0,I_0)=\mathbb{E}\left[\sum_{t=0}^\infty \gamma^t\left\{R(s_t,a_t)+\hat{F}(a^2_t,a^2_{t-1})I_t\right\}\right].\nonumber
\end{align}

\noindent\textbf{\underline{{Summary of events:}}}\\

At a time $t\in 0,1\ldots$ 
    
    
    
    
        
    
    
    \begin{itemize}
        \item Both agents make an observation of the state $s_t\in\mathcal{S}$.
        \item {\fontfamily{cmss}\selectfont Controller} takes an action $a_t$ sampled from its policy $\pi$.
        \item {\fontfamily{cmss}\selectfont Shaper} decides whether or not to activate the shaping reward using $\mathfrak{g}_2:\mathcal{S}  \to \{0,1\}$.
        \item If $\mathfrak{g}_2(s_t)=0$:
        \begin{itemize}
            \item[\textcolor{white}{X}$\circ$] The switch is not activated ($I_{t}=0$). {\fontfamily{cmss}\selectfont Controller} receives a reward $r\sim R(s_t,a_t)$ and the system transitions to the next state $s_{t+1}$.
        \end{itemize}
        \item If $\mathfrak{g}_2(s_t)=1$:
        \begin{itemize}
            \item[\textcolor{white}{X}$\circ$] {\fontfamily{cmss}\selectfont Shaper} takes an action $a_t^2$ sampled from its policy $\pi^2$.
            \item[\textcolor{white}{X}$\circ$] The switch is activated ($I_{t}=1$), {\fontfamily{cmss}\selectfont Controller} receives a reward $R(s_t,a_t)+\hat{F}(a^2_t,a^2_{t-1})\times 1$ and the system transitions to the next state $s_{t+1}$.
        \end{itemize}
    \end{itemize}
We set $\tau_k\equiv 0 \forall k\leq 0$ and  $a^2_k\equiv 0\;\; \forall k\leq 0$ and lastly $a_{\tau_k}^2\equiv 0,\forall k\in\mathbb{N}$ ($a_{\tau_k+1}^2,\ldots, a_{\tau_{k+1}-1}^2$ remain non-zero). The first two conditions ensure the objective is well-defined while the last condition which can be easily ensured, is used in the proof of Prop. \ref{invariance_prop} which guarantees that the optimal policy of the MDP $\mathcal{M}$ is preserved. Lastly, in what follows we use the shorthand $I(t)\equiv I_t$.
%
%
%
\subsection{The Shaper's Objective} 
The goal of the {\fontfamily{cmss}\selectfont Shaper} is to guide {\fontfamily{cmss}\selectfont Controller} to efficiently learn to maximise its own objective. The shaping reward $F$ is activated by switches controlled by  the {\fontfamily{cmss}\selectfont Shaper}.
To induce {\fontfamily{cmss}\selectfont Shaper} to selectively choose when to switch on the shaping reward, each switch activation  incurs a fixed cost for the {\fontfamily{cmss}\selectfont Shaper}. 
This ensures that the gain for the {\fontfamily{cmss}\selectfont Shaper} for encouraging {\fontfamily{cmss}\selectfont Controller} to visit a given set of states is sufficiently high to merit learning optimal shaping reward magnitudes.  Given these remarks the  {\fontfamily{cmss}\selectfont Shaper}'s objective is
%
\begin{align}\nonumber
 v^{\pi,\pi^2}_2(s_0,I_0)  = \mathbb{E}_{\pi,\pi^2}\left[ \sum_{t=0}^\infty \gamma^t\left(\hat{R}_1 -\sum_{k\geq 1}^\infty \delta^t_{\tau_{2k-1}}
+L(s_t)\right)\right],& 
 \end{align}
%
%
%
where $\delta^t_{\tau_{2k-1}}$ is the Kronecker-delta function which introduces a cost for each switch, is $1$ whenever $t={\tau_{2k-1}}$ and $0$ otherwise (this restricts the costs to only the points at which the shaping reward is activated). The term $L$ is a {\fontfamily{cmss}\selectfont Shaper} \textit{bonus reward} for when the {\fontfamily{cmss}\selectfont Controller} visits an infrequently visited state and tends to $0$ as the state is revisited.

The objective encodes the {\fontfamily{cmss}\selectfont Shaper}'s agenda, namely to maximise the expected return.\footnote{Note that we can now see that $\hat{R}_2\equiv R(s_t,a_t)+\hat{F}(a^2_t,a^2_{t-1})I_t-\sum_{k\geq 1}^\infty \delta^t_{\tau_{2k-1}}+L(s_t)$.} Therefore, using its shaping rewards, the  {\fontfamily{cmss}\selectfont Shaper} seeks to guide {\fontfamily{cmss}\selectfont Controller} towards optimal trajectories (potentially away from suboptimal trajectories, c.f. Experiment 1) and enable {\fontfamily{cmss}\selectfont Controller} to learn faster (c.f. Cartpole experiment in Sec. \ref{Section:Experiments}). 
With this, the {\fontfamily{cmss}\selectfont Shaper} constructs a shaping-reward function that supports the {\fontfamily{cmss}\selectfont Controller}'s learning which is tailored for the specific setting. This avoids inserting hand-designed exploration heuristics into the {\fontfamily{cmss}\selectfont Controller}'s objective as in curiosity-based methods \cite{burda2018exploration,pathak2017curiosity} and classical reward shaping \cite{ng1999policy}. 
We later prove that with this objective, the {\fontfamily{cmss}\selectfont Shaper}'s optimal policy maximises {\fontfamily{cmss}\selectfont Controller}'s (extrinsic) return (Prop. \ref{invariance_prop}). Additionally, we show that the framework preserves the optimal policy of $\mathfrak{M}$.     

\subsection*{Discussion on Shaper Bonus Term $L$}
For this there are various possibilities e.g. model prediction error \cite{stadie2015incentivizing}, count-based exploration bonus \cite{strehl2008analysis}. We later show our method performs well regardless of the choice of bonus rewards and outperforms RL methods in which these bonuses are added to the agent's objective directly (see Sec. \ref{sec:ablation_studies}).
\subsection*{Discussion on Computational Aspect}
The switching control mechanism results in a framework in which the problem facing the {\fontfamily{cmss}\selectfont Shaper} has a markedly reduced decision space in comparison to the {\fontfamily{cmss}\selectfont Controller}'s problem  (though both share the same experiences). Crucially, the {\fontfamily{cmss}\selectfont Shaper} must compute optimal shaping rewards at only a subset of states which are chosen by $\mathfrak{g}_2$. Moreover, the decision space for the switching policy $\mathfrak{g}_2$ is $\mathcal{S}\times\{0,1\}$ i.e at each state it makes a binary decision. Consequently, the learning process for $\mathfrak{g}_2$ is much quicker than the {\fontfamily{cmss}\selectfont Controller}'s policy which must optimise over a decision space which is $|\mathcal{S}||\mathcal{A}|$ (choosing an action from its action space at every state). This results in the {\fontfamily{cmss}\selectfont Shaper} rapidly learning its optimal policies (relative to the {\fontfamily{cmss}\selectfont Controller}) in turn, enabling the {\fontfamily{cmss}\selectfont Shaper} to guide the {\fontfamily{cmss}\selectfont Controller} towards its optimal policy during its learning phase. Additionally, in our experiments, we chose the size of the action set for the {\fontfamily{cmss}\selectfont Shaper}, $\mathcal{A}_2$ to be a singleton resulting in a decision space of size $|\mathcal{S}|\times\{0,1\}$ for the entire problem facing the {\fontfamily{cmss}\selectfont Shaper}. We later show that this choice leads to improved performance while removing the free choice of the dimensionality of the {\fontfamily{cmss}\selectfont Shaper}'s action set. Lastly, we later prove that the optimal policy for the {\fontfamily{cmss}\selectfont Shaper} maximises the {\fontfamily{cmss}\selectfont Controller}'s objective (Prop. \ref{invariance_prop}).
\subsection{The Overall Learning Procedure} 
The game $\mathcal{G}$ is solved using our multi-agent RL algorithm (ROSA). In the next section, we show the convergence properties of ROSA. The full code is in Sec. \ref{sec:algorithm} of the Appendix. 
The ROSA algorithm consists of two independent procedures: {\fontfamily{cmss}\selectfont Controller} learns its own policy while {\fontfamily{cmss}\selectfont Shaper} learns which states to perform a switch and the shaping reward magnitudes.
In our implementation, we used proximal policy optimization (PPO) \cite{schulman2017Proximal} as the learning algorithm for all policies: {\fontfamily{cmss}\selectfont Controller}'s policy, switching control policy, and the reward magnitude policy.
We demonstrated ROSA with various {\fontfamily{cmss}\selectfont Shaper} $L$ terms, the first is RND \cite{burda2018exploration} in which $L$ takes the form  $L(s_{t}):=\|\hat{h}(s_t) - \mathit{h(s_t)}\|_{2}^{2}$ where $h$ is a random initialised, fixed target network while $\hat{h}$ is the predictor network that seeks to approximate the target network. Secondly, to demonstrate the flexibility of ROSA to perform well with even a rudimentary bonus term, we use a simple count-based term for $L$, which counts the number of times a state has been visited  (see Sec. \ref{sec:ablation_studies}).   
The action set of the {\fontfamily{cmss}\selectfont Shaper} is thus $\mathcal{A}_2 \coloneqq \{0, 1,..., m\}$ where each element is an element of $\mathbb{N}$, and $\pi_2$ is a MLP $\pi_2: \mathbb{R}^d    \mapsto \mathbb{R}^m$. Precise details are in the Supplementary Material, Section 8.

\section{Convergence and Optimality of ROSA} 
The ROSA framework enables the {\fontfamily{cmss}\selectfont Shaper} to
learn a shaping-reward function to assist the {\fontfamily{cmss}\selectfont Controller} when learning a (near-)optimal policy. The interaction between the two RL agents induces two concurrent learning processes, potentially raisingconvergence issues \cite{zinkevich2006cyclic}. We now show that ROSA converges and that the performance of the resulting policy is similar to solving $\mathfrak{M}$ directly. To achieve this, we first study the stable point solutions of $\mathcal{G}$.  
%
Unlike MDPs, the existence of a stable point solutions in Markov policies is not guaranteed for MGs \cite{blackwell1968big} and are rarely computable.\footnote{Special exceptions are \textit{team} MGs where agents share an objective and \textit{zero-sum} MGs \cite{shoham2008multiagent}.}
MGs also often have multiple stable points that can be inefficient \cite{mguni2019coordinating}; in $\mathcal{G}$, the outcome of such stable point profiles may be a poor performing {\fontfamily{cmss}\selectfont Controller} policy. To ensure the framework is useful, we must verify that the solution of $\mathcal{G}$ corresponds to $\mathfrak{M}$. 
We address the following challenges:

%
$\textbf{1.}$ ROSA preserves the optimal policy of $\mathfrak{M}$.

$\textbf{2.}$ A stable point of the game $\mathcal{G}$ in Markov policies exists.

$\textbf{3.}$ ROSA converges to the stable point solution of $\mathcal{G}$.

$\textbf{4.}$ The convergence point of ROSA yields a payoff that is (weakly) greater than that from solving $\mathfrak{M}$ directly.

In proving 1--4 we deduce the following:
\begin{theorem}\label{main_theorem}
ROSA ensures conditions C.1 and C.2.  
\end{theorem}
Proofs are deferred to the Appendix. 

We now give our first result that shows the solution to $\mathfrak{M}$ is preserved under the influence of the {\fontfamily{cmss}\selectfont Shaper}: 
\begin{proposition}\label{invariance_prop}The following statements hold $\forall s\in\mathcal{S}$:
\begin{itemize}
    \item [i)] $
\underset{\pi\in\Pi}{\arg\max}\; \tilde{v}^{\pi,\pi^2}(s)=\underset{\pi\in\Pi}{\arg\max}\; v^{\pi}(s),  \forall \pi^2\in\Pi^2$, 
\hspace{-8 mm}\item [ii)] The {\fontfamily{cmss}\selectfont Shaper}'s optimal policy maximises $v^{\pi}(s)$. 
\end{itemize}
\end{proposition}
Recall, $v^{\pi}$ denotes the {\fontfamily{cmss}\selectfont Controller}'s expected return without the influence of the {\fontfamily{cmss}\selectfont Shaper}. Result (i) therefore says that the {\fontfamily{cmss}\selectfont Controller}'s problem is preserved under the influence of the {\fontfamily{cmss}\selectfont Shaper}. Moreover the (expected) total return received by the {\fontfamily{cmss}\selectfont Controller} is that from the environment. Result (ii) establishes that the {\fontfamily{cmss}\selectfont Shaper}'s optimal policy induces {\fontfamily{cmss}\selectfont Shaper} to maximise {\fontfamily{cmss}\selectfont Controller}'s extrinsic total return.  

The result comes from a careful adaptation of the policy invariance result \cite{ng1999policy} to our multi-agent switching control framework, where the shaping reward is no longer added at all states. 
%
Building on Prop. \ref{invariance_prop}, we find:
\begin{corollary}\label{invariance_prop_corollary}
ROSA preserves the MDP $\mathfrak{M}$. In particular, let $(\hat{\pi}^1,\hat{\pi}^2)$ be a stable point policy profile\footnote{By stable point profile we mean a Markov perfect equilibrium \cite{fudenberg1991tirole}.} of the MG induced by ROSA $\mathcal{G}$. Then, $\hat{\pi}^1$ is a solution to the MDP, $\mathfrak{M}$. 
\end{corollary}
Hence, introducing the {\fontfamily{cmss}\selectfont Shaper} does not alter the solution. 

We next show that the solution of $\mathcal{G}$ can be computed as a limit point of a sequence of Bellman operations.  We use this to show the convergence of ROSA. We define a \textit{projection} $\cP$ on a function $\Lambda$ by: $
\cP \Lambda:=\underset{\bar{\Lambda}\in\{\Psi r|r\in\mathbb{R}^p\}}{\arg\min}\left\|\bar{\Lambda}-\Lambda\right\|$:

\begin{theorem}\label{theorem:existence}
\textbf{i)} Let $V:\mathcal{S}\times\mathbb{N}\to\mathbb{R}$ then the game $\mathcal{G}$ has a stable point which is a given by $
\underset{k\to\infty}{\lim}T^kV^{\boldsymbol{\pi}}=\underset{{\boldsymbol{\hat{\pi}}}\in\boldsymbol{\Pi}}{\sup}V^{\boldsymbol{\hat{\pi}}}=V^{\boldsymbol{\pi^\star}}$, where $\boldsymbol{\hat{\pi}}$ is a stable policy profile for the MG, $\mathcal{G}$ and $T$ is the Bellman operator of $\mathcal{G}$.

\textbf{ii)} ROSA converges to the stable point of $\mathcal{G}$. Moreover, given a set of linearly independent basis functions $\Psi=\{\psi_1,\ldots,\psi_p\}$ with $\psi_k\in L_2,\forall k$, ROSA converges to a limit point $r^\star\in\mathbb{R}^p$ that is the unique solution to  $\cP \mathfrak{F} (\Psi r^\star)=\Psi r^\star$, where $\mathfrak{F}$ is defined by:
    $\mathfrak{F}\Lambda:=\hat{R}_1+\gamma P \max\{\mathcal{M}\Lambda,\Lambda\}$ where $\cM$ is defined by $\mathcal{M}^{\pi,\pi^2}v(s,\cdot):=\hat{R}_1-1+\gamma\sum_{s'\in\mathcal{S}}P(s';a_{\tau_k},s)v(s',\cdot)|a_{\tau_k}\sim \pi^2$  and $r^\star$ satisfies: $
    \left\|\Psi r^\star - Q^\star\right\|\leq (1-\gamma^2)^{-1/2}\left\|\cP Q^\star-Q^\star\right\|$.
\end{theorem}
Part i) of the theorem proves the system in which the {\fontfamily{cmss}\selectfont Shaper} and {\fontfamily{cmss}\selectfont Controller} jointly learn has a stable point and is the limit of a dynamic programming procedure. Crucially (by Corollary \ref{invariance_prop_corollary}), the limit point corresponds to the solution of the MDP $\mathcal{M}$.  This is proven by showing that $\mathcal{G}$ has a dual representation as an MDP whose solution corresponds to the stable point of the MG. This then enables a distributed Q-learning method \cite{bertsekas2012approximate} to tractably solve $\mathcal{G}$.

Part ii) establishes the solution to $\mathcal{G}$ can be computed using ROSA. This means that the {\fontfamily{cmss}\selectfont Shaper} converges to a shaping-reward function and (by Prop. \ref{invariance_prop}) the {\fontfamily{cmss}\selectfont Controller} learns the optimal value function for $\mathcal{M}$.
The result also establishes the convergence of ROSA to the solution using (linear) function approximators and bounds the approximation error by the smallest error achievable (given the basis functions).
%

%
%
%

%
Introducing poor shaping rewards can potentially worsen overall performance. We now prove ROSA introduces shaping rewards that yield higher total environment returns for the {\fontfamily{cmss}\selectfont Controller}, as compared to solving $\mathfrak{M}$ directly. 

\begin{proposition}\label{NE_improve_prop}
There exists some finite integer $N$ such that $v^{\tilde{\pi}_m}(s)\geq v^{\pi_m}(s),\;\forall s \in\mathcal{S}$ for any $m\geq N$, where $\tilde{\pi}_m$ and $\pi_m$ are the respective {\fontfamily{cmss}\selectfont Controller} policies after the $m^{th}$ learning iteration with and without the {\fontfamily{cmss}\selectfont Shaper}'s influence.
\end{proposition}

Note that Prop. \ref{NE_improve_prop} implies $v^{\tilde{\pi}}(s)\geq v^{\pi}(s),\;\forall s \in\mathcal{S}$. Prop. \ref{NE_improve_prop} shows that the {\fontfamily{cmss}\selectfont Shaper} improves outcomes for the {\fontfamily{cmss}\selectfont Controller}. Additionally, unlike reward shaping methods in general, the shaping rewards generated by  the {\fontfamily{cmss}\selectfont Shaper} \textit{never} lead to a reduction to the total (environmental) return for Controller (compared to the total return without $F$). 

\textbf{Note:} \textit{Prop. \ref{NE_improve_prop} compares the environmental (extrinsic) rewards accrued by the {\fontfamily{cmss}\selectfont Controller}. Prop. \ref{NE_improve_prop} therefore shows {\fontfamily{cmss}\selectfont Shaper} induces a {\fontfamily{cmss}\selectfont Controller} policy that leads to a (weakly) higher expected return from the \textit{environment}.}

\section{Experiments}\label{Section:Experiments}
We performed a series of experiments to test if ROSA \textbf{(1)} learns beneficial shaping-reward functions \textbf{(2)} decomposes complex tasks, and \textbf{(3)} tailors shaping rewards to encourage the {\fontfamily{cmss}\selectfont Controller} to capture environment rewards (as opposed to merely pursuing novelty). We compared ROSA's performance to RND \cite{burda2018exploration}, ICM \cite{pathak2017curiosity}, LIRPG \cite{zheng2018learning}, BiPaRS-IMGL \cite{hu2020learning}\footnote{BiPaRS-IMGL requires a manually crafted shaping-reward (only available in Cartpole).} and vanilla PPO \cite{schulman2017Proximal}.
We then compared performances on performance benchmarks including Sparse Cartpole, Gravitar, Solaris, and Super Mario. 
\subsection{Didatical Examples}



\begin{figure}[ht]
    \centering
    \includegraphics[width=0.95\linewidth]{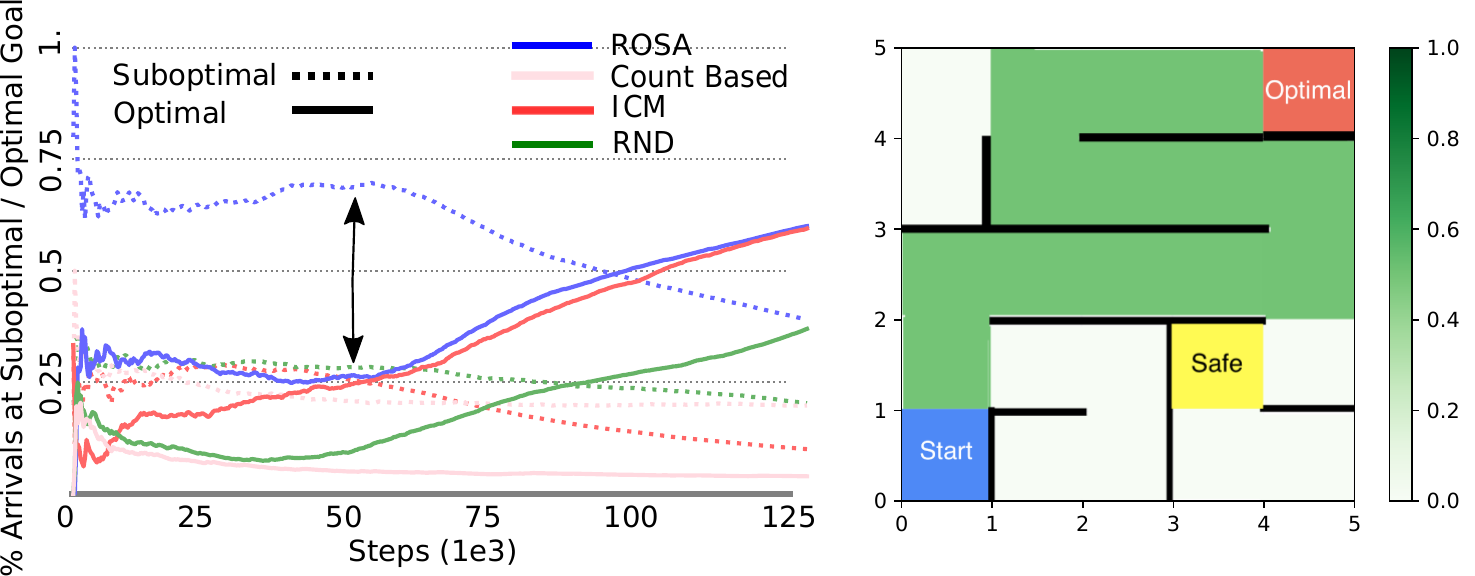}
    \caption{\textit{Left.} Proportion of optimal and suboptimal goal arrivals. ROSA has a marked inflection (arrow) where arrivals at the sub-optimal goal decrease and arrivals at the optimal goal increase. {\fontfamily{cmss}\selectfont Shaper} has learned to guide {\fontfamily{cmss}\selectfont Controller} to forgo the suboptimal goal in favour of the optimal one. \textit{Right.} Heatmap showing where ROSA adds rewards.}
    \label{Figure:explore_Maze}
    \vspace{-3mm}
\end{figure}
\noindent\textbf{Beneficial shaping reward.} 
ROSA is able to learn a shaping-reward function that leads to improved {\fontfamily{cmss}\selectfont Controller} performance. 
\begin{figure}[ht]
    \centering
    \includegraphics[width=8cm, height=3.5cm]{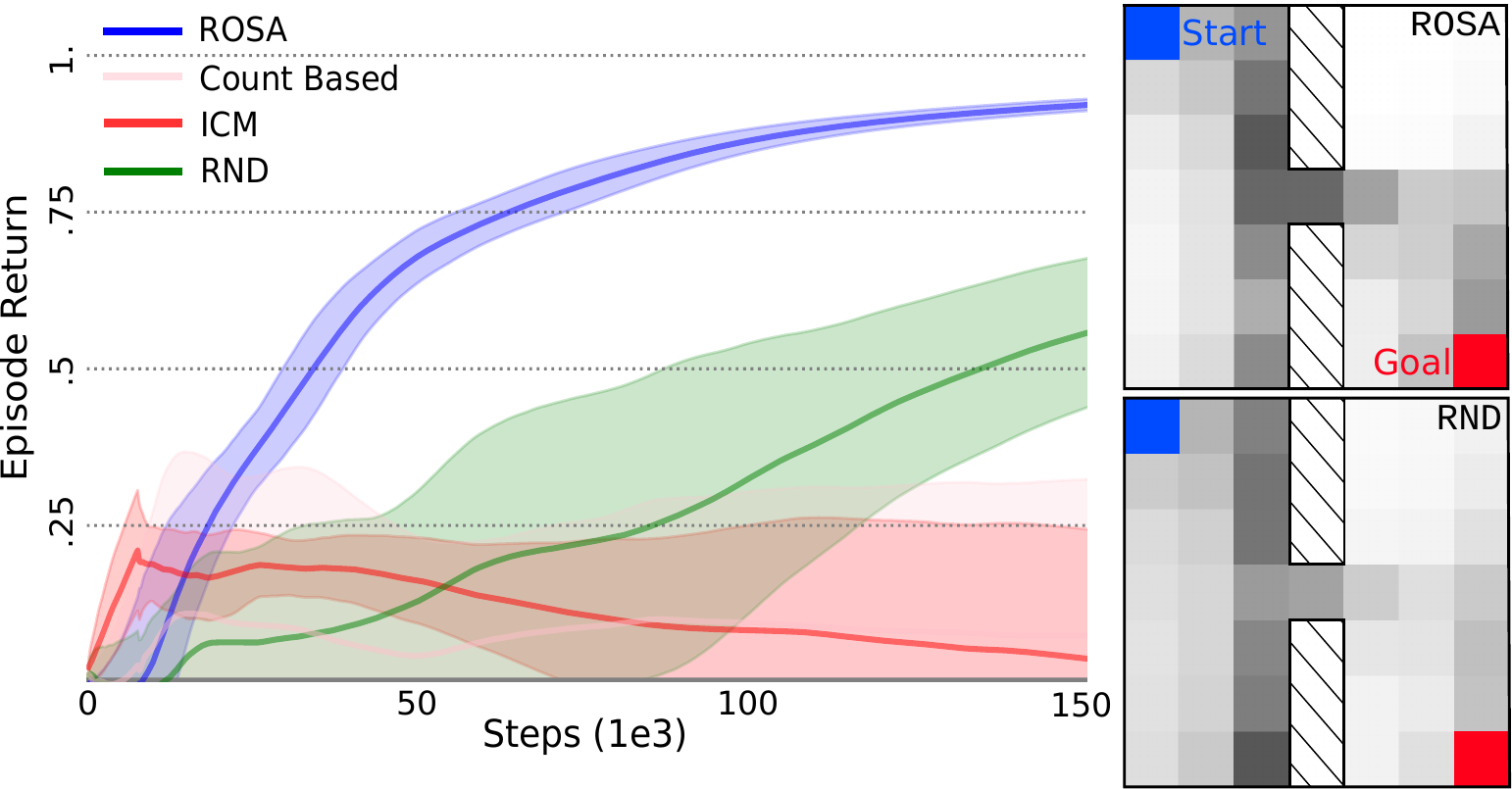}
    \caption{Discovering subgoals on Subgoal Maze. \textit{Left.} Learning curves. \textit{Right.} Heatmap of shaping rewards.} 
    \label{Figure:Subgoal_Maze}
    \vspace{-3mm}
\end{figure}In particular, it is able to learn to shape rewards that encourage the RL agent to avoid suboptimal --- but easy to learn --- policies in favour of policies that attain the maximal return. To demonstrate this, we designed a Maze environment with two terminal states: a suboptimal goal state that yields a reward of $0.5$ and an optimal goal state which yields a reward of $1$. In this maze design, the sub-optimal goal is more easily reached. A good shaping-reward function discourages the agent from visiting the sub-optimal goal. As shown in Fig. \ref{Figure:explore_Maze}\footnote{The sum of curves for each method may be less that 1 if the agent fails to arrive at either goal.} ROSA achieves this by learning to place shaping rewards (dark green) on the path that leads to the optimal goal. 

\subsection*{Subgoal discovery.} 
We used the Subgoal Maze introduced in  \cite{mcgovern2001automatic} to test if ROSA can discover subgoals. The environment has two rooms separated by a gateway. To solve this, the agent must discover the subgoal (reaching the gateway before it can reach the goal. Rewards are $-0.01$ everywhere except at the goal state where the reward is $1$. As shown in Fig. \ref{Figure:Subgoal_Maze}, ROSA successfully solves this environment whereas other methods fail. ROSA assigns importance to reaching the gateway, depicted by the heatmap of added shaped rewards. 

\begin{figure}[ht]
    \centering
    \includegraphics[width = 0.47 \textwidth]{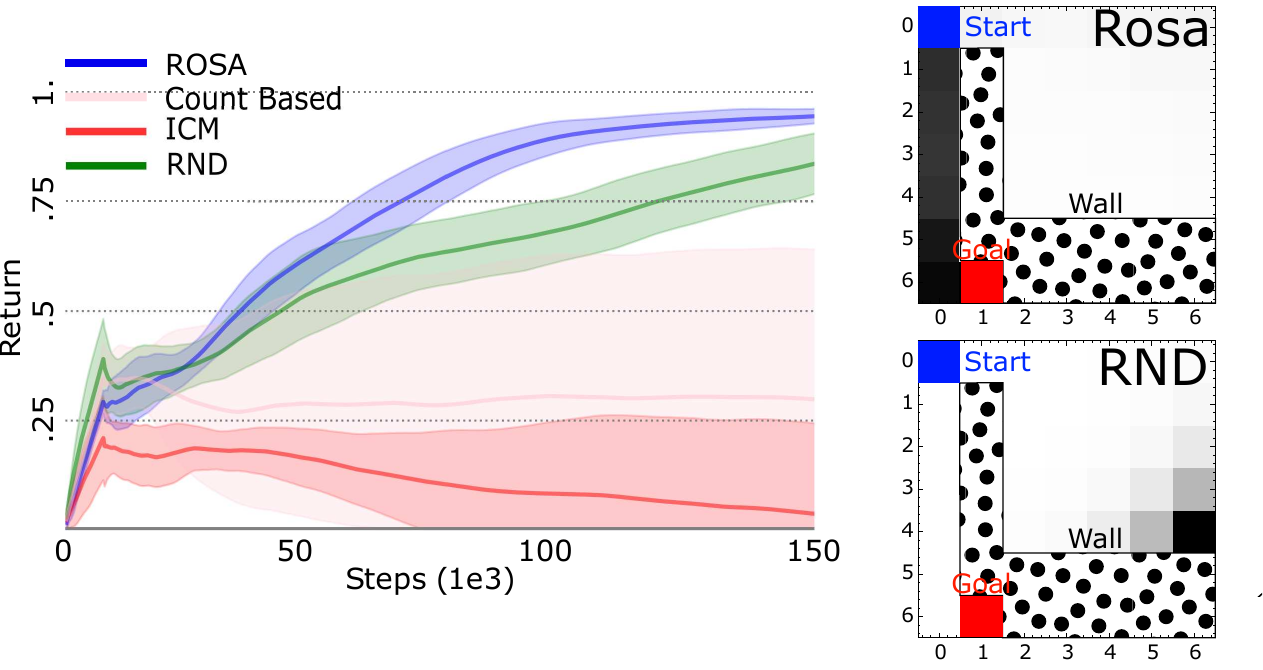}
    \caption{Red-Herring Maze. Ignoring non-beneficial shaping reward. \textit{Left.} Learning curves. \textit{Right.} Heatmap of added shaping rewards. ROSA ignores the RHS of the maze, while RND incorrectly adds unuseful shaping rewards there.}
    \vspace{-4mm}
    \label{Figure:red_herring_Maze}
\end{figure}
\noindent\textbf{Ignoring non-beneficial shaping reward.} 
Switching control gives ROSA the power to learn when to attend to shaping rewards and when to ignore them. This allows us to learn to ignore ``red-herrings'', i.e., unexplored parts of the state space where there is no real environment reward, but where surprise or novelty metrics would place high shaping rewards. To verify this claim, we use a modified Maze environment called Red-Herring Maze in which a large part of the state space that has no environment reward, but with the goal and environment reward elsewhere. Ideally, we expect that the reward shaping method can learn to quickly ignore the large part of the state space. Fig. \ref{Figure:red_herring_Maze} shows ROSA outperforms all other baselines. Moreover, the heatmap shows that while RND is easily dragged to reward exploring novel but non rewarding states, ROSA learns to ignore them.
\begin{figure}[ht]
    \centering
    \includegraphics[width = 0.47\textwidth]{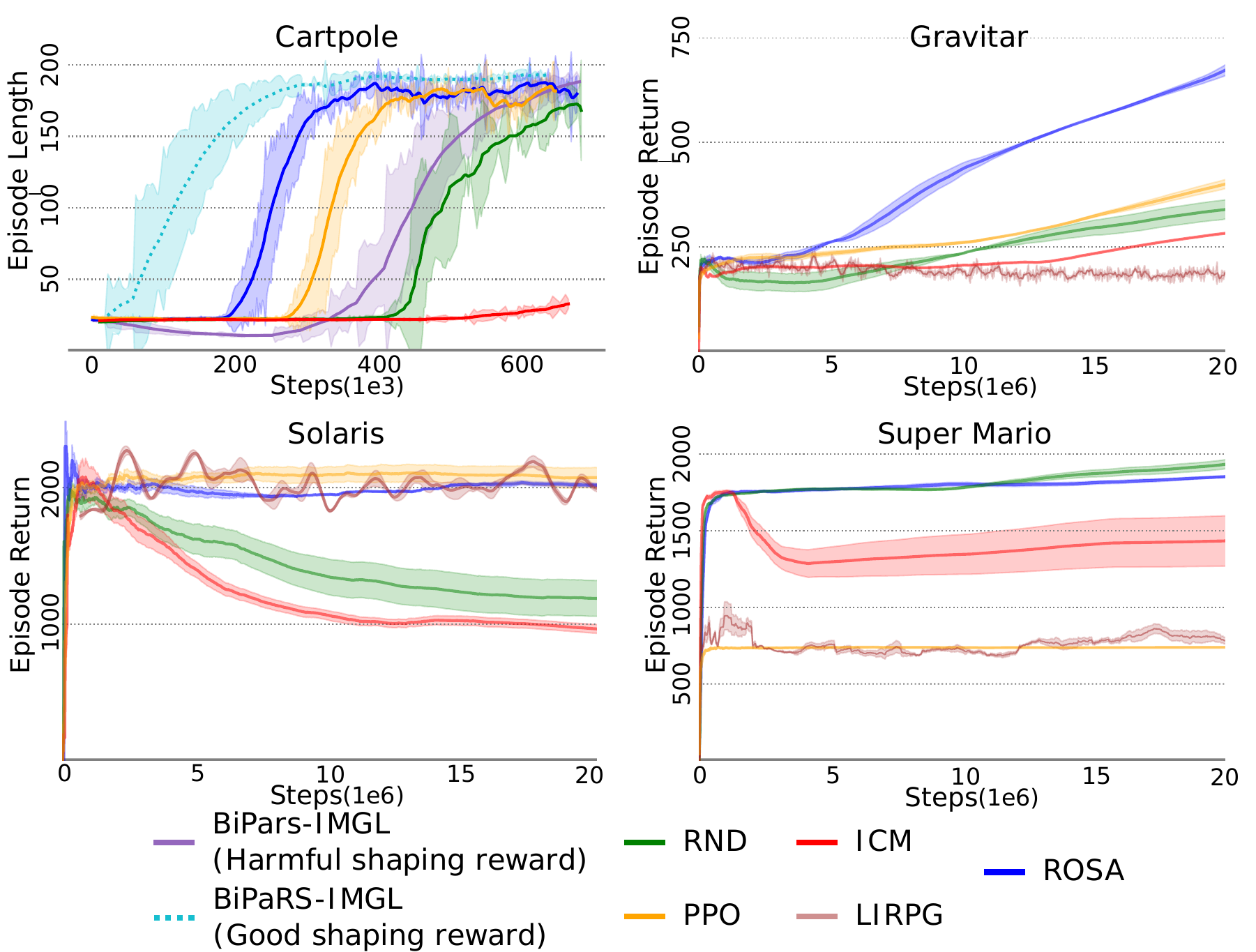}
    \caption{Benchmark performance.}
    \label{Figure:LearningCurves}
    \vspace{-5mm}
\end{figure}

\subsection{Learning Performance.} We compared ROSA with the baselines in four challenging sparse rewards environments: Cartpole, Gravitar, Solaris, and Super Mario. These environments vary in state representation, transition dynamics and reward sparsity. In Cartpole, a penalty of $-1$ is received only when the pole collapses; in Super Mario Brothers the agent can go for 100s of steps without encountering a reward. Fig. \ref{Figure:LearningCurves} shows learning curves. ROSA either markedly outperforms the best competing baseline (Cartpole and Gravitar) or is on par with them (Solaris and Super Mario) showing that it is robust to the nature of the environment and underlying sparse reward. Moreover, ROSA does not exhibit the  failure modes where after good initial performance it deteriorates. E.g., in Solaris both ICM and RND have good initial performance but deteriorate sharply while ROSA's performance remains satisfactory.
\section{Ablation Studies}\label{sec:ablation_studies}

To understand how ROSA's performance is affected by components of  the algorithm or hyper-parameter settings, we ran a series of ablation experiments. All experiments in this section were run on a simple $25$x$25$ Gridworld in Fig.~\ref{fig:ablation_25x25_simple_maze}.
\noindent\textbf{ROSA is an effective plug \& play framework.} 
Fig. \ref{fig:ablation_plug_n_play} show the performance of vanilla PPO and vanilla TRPO versus their ROSA enhanced counterparts. Particularly notable is ROSA's enhancement to TRPO. Both vanilla TRPO and TRPO+ROSA perform equally well in early stages of learning, but while vanilla TRPO seems to get stuck with a suboptimal policy, TRPO+ROSA consistently improves performance through learning until reaching convergence.

\begin{figure}[ht]
    \begin{subfigure}{0.27\columnwidth}
        \centering
        \includegraphics[width=.75\columnwidth]{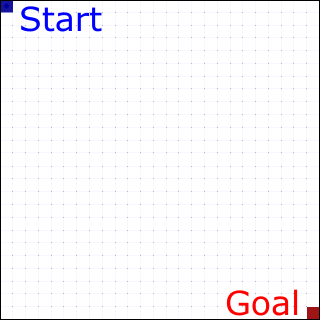}
        \caption{Gridworld.}
        \label{fig:ablation_25x25_simple_maze}
    \end{subfigure}
    \begin{subfigure}{0.73\columnwidth}
        \includegraphics[width=1.0\columnwidth]{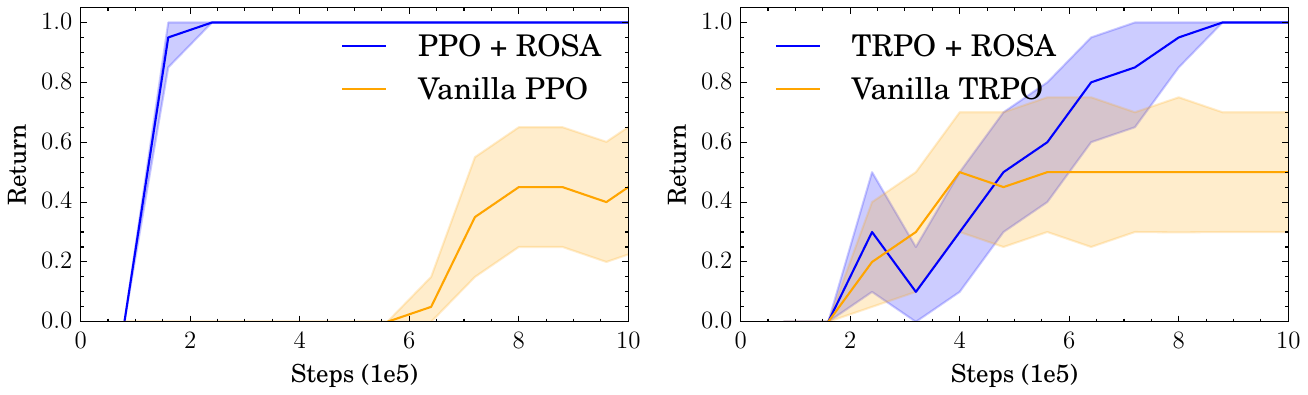}
        \caption{ROSA is an effective plug \& play framework. Enhancing both PPO (left) and TRPO (right) with ROSA results in marked performance gains.}
        \label{fig:ablation_plug_n_play}
    \end{subfigure}
\end{figure}
\noindent\textbf{ROSA delivers performance boost despite severe impairments to the method.}
A core component of ROSA is the exploration bonus term in the {\fontfamily{cmss}\selectfont Shaper}'s objective. We ran experiments to check if ROSA can still deliver a performance boost when this important component of the algorithm is either weakened or ablated out entirely. Fig. \ref{fig:ablation_L} shows performance of various versions of ROSA: one with RND providing the exploration bonus, one with a simple count-based measure $L(s)=\frac{1}{\text{Count(s)}+1}$ providing the exploration bonus, and one where the exploration bonus is ablated out entirely. Stronger exploration bonuses such as RND enable ROSA (PPO+ROSA ($L$=RND)) to provide more effective reward shaping over weaker exploration bonuses (PPO+ROSA ($L$=Count-based)), indicating this is an important aspect. Yet, ROSA can still usefully benefit learners even when the exploration bonus is ablated completely as shown by the fact that (PPO+ROSA (No. $L$)) outperforms Vanilla PPO.

\begin{figure}[ht]
    \centering
    \includegraphics[width=8cm, height=3.5cm]{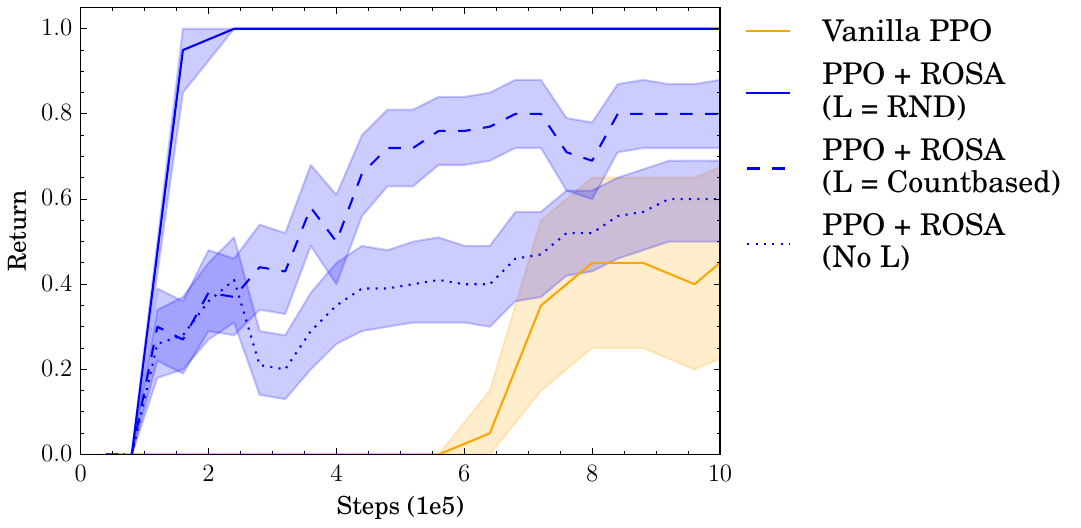}
    \caption{Ablation study on the exploration bonus. }
    \label{fig:ablation_L}
\end{figure}
\noindent\textbf{Tuning switching cost is important.}
Switching cost is a fundamental component of switching contol methods. Lower switching costs allow for ROSA to be less discriminatory about where it adds shaping reward, while higher switching costs may prevent ROSA from adding useful shaping rewards. Thus, for each environment this hyper-parameter must be tuned to obtain optimal performance from ROSA. Fig. \ref{fig:ablation_switching_cost} shows a parameter study of end-of-training evaluation performance of PPO+ROSA versus various values for the switching cost. As can be seen, outside of an optimum range of values for the switching cost, approximately $(-0.1,-0.01)$, ROSA's effectiveness is hampered. It is future work to investigate how this hyper-parameter should be set.

\begin{figure}[ht]
    \centering
    \includegraphics[width=5cm, height=3cm]{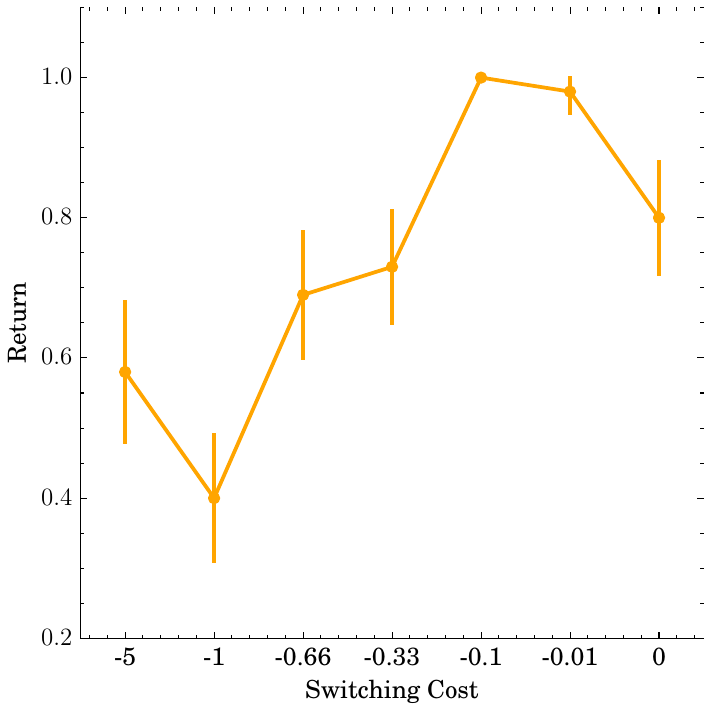}
    \caption{Switching cost.}
    \label{fig:ablation_switching_cost}
    \vspace{-3 mm}
\end{figure}


\section{Conclusion}\label{Section:Conclusion}
We presented a novel solution method to solve the problem of reward shaping. Our Markov game framework of a primary Controller and a secondary reward shaping agent is guaranteed to preserve the underlying learning task for the {\fontfamily{cmss}\selectfont Controller} whilst guiding {\fontfamily{cmss}\selectfont Controller} to higher performance policies. Moreover, ROSA is able to decompose complex learning tasks into subgoals and to adaptively guide {\fontfamily{cmss}\selectfont Controller} by selectively choosing the states to add shaping rewards. By presenting a theoretically sound and empirically robust approach to solving the reward shaping problem, ROSA opens up the applicability of RL to a range of real-world control problems. The most significant contribution of this paper, however, is the novel construction that marries RL, multi-agent RL and game theory which leads to a new solution method in RL. We believe this powerful approach can be adopted to solve other open challenges in RL.

\clearpage
\bibliography{main}
\clearpage
\appendix

\renewcommand*{\thesection}{\arabic{section}}
\setcounter{section}{7}
\setcounter{page}{1}

\addcontentsline{toc}{section}{Appendix} 
\part{{\Large{Appendix}}} 
\parttoc
\clearpage


\footnotesize
\onecolumn
\section{Algorithm}\label{sec:algorithm}

\begin{algorithm}[H]
    \label{algo:Opt_reward_shap} 
    \DontPrintSemicolon
    \KwInput{ Environment $E$ \;
             \hspace{3em} Initial {\fontfamily{cmss}\selectfont Controller} policy $\pi_0$ with parameters  $\theta_{\pi_0}$ \;  
             \hspace{3em} Initial {\fontfamily{cmss}\selectfont Shaper} switch policy $\mathfrak{g}_{2_{0}}$ with parameters $\theta_{\mathfrak{g}_{2_{0}}}$\; 
             \hspace{3em} Initial {\fontfamily{cmss}\selectfont Shaper} action policy $\pi^2_0$ with parameters $\theta_{\pi^2_0}$ \;
             \hspace{3em} Neural networks $h$ (fixed) and $\hat{h}$ for RND with parameter $\theta_{\hat{h}}$\; 
             \hspace{3em} Buffer $B$ \;
             \hspace{3em} Number of rollouts $N_r$, rollout length $T$ \;
             \hspace{3em} Number of mini-batch updates $N_u$ \;
             \hspace{3em} Switch cost $c(\cdot)$, Discount factor $\gamma$, learning rate $\alpha$\;
             
             }
    \KwOutput{Optimised {\fontfamily{cmss}\selectfont Controller} policy $\pi^*$}
    $\pi, \pi^2, \mathfrak{g}_{2} \gets \pi_0, \pi^2_0,\mathfrak{g}_{2_{0}}$\;
    \For{$n = 1, N_r$}
    {
        \textbf{// Collect rollouts}\;
        \For{$t = 1, T$}
        {
            Get environment states $s_t$ from $E$ \;
            Sample $a_t$ from $\pi(s_t)$ \;
            Apply action $a_t$ to environment $E$, and get reward $r_t$ and next state $s_{t+1}$ \;
            Sample $g_t$ from $\mathfrak{g}_{2}(s_t)$ \textbf{ // Switching control} \;
            \eIf{$g_t = 1$}
            {   
                Sample $a^2_t$ from $\pi^2(s_t)$ \;
                Sample $a^2_{t+1}$ from $\pi^2(s_{t+1})$ \;
                $r^i_t = \gamma a^2_{t+1} -  a^2_{t}$ \textbf{// Calculate $F(a^2_t, a^2_{t+1})$}\;
            }
            {   
                $a^2_t, r^i_t = 0, 0$ \textbf{// Dummy values}
            }
            Append $(s_t, a_t, g_t, a^2_t, r_t, r^i_t, s_{t+1})$ to $B$
        }
        \For{$u = 1, N_u$}
        {
            Sample data $(s_t, a_t, g_t, a^2_t, r_t, r_t^i, s_{t+1})$ from $B$\;
            \eIf{$g_t = 1$}
            {   
                Set shaped reward to $r_t^s = r_t + r_t^i$
            }
            {
                Set shaped reward to $r_t^s = r_t$
            }
            \textbf{// Update RND} \;
            $\text{Loss}_{\text{RND}} = ||h(s_t) - \hat{h}(s_t)||^2$ \;
            $\theta_{\hat{h}} \gets \theta_{\hat{h}} - \alpha \nabla \text{Loss}_{\text{RND}} $ \;
            \textbf{// Update {\fontfamily{cmss}\selectfont Shaper} }\;
            $l_t = ||h(s_t) - \hat{h}(s_t)||^2$ \textbf{// Compute $L(s_t)$} \;
            $c_t =  g_t$ \;
            Compute $\text{Loss}_{\pi^2}$ using $(s_t, a_t, g_t, c_t, r_t, r_t^i, l_t, s_{t+1})$ using PPO loss \textbf{// Section 4.2} \; 
            Compute $\text{Loss}_{ \mathfrak{g}_{2}}$ using $(s_t, a_t, g_t, c_t, r_t, r_t^i, l_t, s_{t+1})$ using PPO loss \textbf{// Section 4.2}\;
            $\theta_{\pi^2} \gets \theta_{\pi^2} - \alpha \nabla \text{Loss}_{\pi^2}$ \;
            $\theta_{\mathfrak{g}_{2}} \gets \theta_{\mathfrak{g}_{2}} - \alpha \nabla \text{Loss}_{\mathfrak{g}_{2}}$ \;
            \textbf{// Update \fontfamily{cmss}\selectfont Controller}\;
            Compute $\text{Loss}_{ \pi}$ using $(s_t, a_t, r_t^s, s_{t+1})$ using PPO loss \textbf{// Section 4} \;
            $\theta_{\pi} \gets \theta_{\pi} - \alpha \nabla \text{Loss}_{\pi}$ \;
        }
    }
	\caption{\textbf{R}einforcement Learning \textbf{O}ptimising \textbf{S}haping \textbf{A}lgorithm ROSA}
\end{algorithm}
\normalsize
\clearpage
\normalsize

\section{Further Implementation Details}\label{sec:app_imp_details}


Details of the {\fontfamily{cmss}\selectfont Shaper} and $F$ (shaping reward)\\
\begin{tabular}{c|l}
\textbf{Object} & \textbf{Description}\\
\hline
    & [512, \texttt{ReLU}, 512, \texttt{ReLU}, 512, $m$] \\
$\mathcal{A}_2$ & Discrete integer action set which is size of output of $f$,\\
    & i.e.,$\mathcal{A}_2$ is set of integers $\{1,...,m\}$ \\
$\pi_2$ & Fixed feed forward NN that maps $\mathbb{R}^d \mapsto \mathbb{R}^m$ \\
    & [512, \texttt{ReLU}, 512, \texttt{ReLU}, 512, $m$] \\
$F$ & $\gamma  a^2_{t+1}$ - $a^2_{t},$\;\; $\gamma=0.95$\\
\end{tabular}

$d$=Dimensionality of states; $m\in \mathbb{N}$ - tunable free parameter.

\clearpage
\section{Experimental Details}\label{sec:experimental_details}
\subsection{Environments \& Preprocessing Details}
The table below shows the provenance of environments used in our experiments.

\begin{center}
    \begin{tabular}{r|l} 
        \toprule
        Atari \& Cartpole & \url{https://github.com/openai/gym}\\
        Maze & \url{https://github.com/MattChanTK/gym-maze} \\
        Super Mario Brothers & \url{https://github.com/Kautenja/gym-super-mario-bros}\\
        \bottomrule
    \end{tabular}
    \label{Table:Environments}
\end{center}

Furthermore, we used preprocessing settings as indicated in the following table.
\small
\begin{center}
    \begin{tabular}{r|l} 
        \toprule
        Setting & Value \\
        \midrule
        Max frames per episode & Atari \& Mario $\rightarrow$ 18000 / Maze \& Cartpole $\rightarrow$ 200\\
        Observation concatenation & Preceding 4 observations\\
        Observation preprocessing & Standardization followed by clipping to [-5, 5]\\
        Observation scaling & Atari \& Mario $\rightarrow$ (84, 84, 1) / Maze \& Cartpole $\rightarrow$ None \\
        Reward (extrinsic and intrinsic) preprocessing & Standardization followed by clipping to [-1, 1]\\
        \bottomrule
    \end{tabular}
    \label{Table:Environments}
\end{center}

\normalsize
\clearpage
\subsection{Hyperparameter Settings}
In the table below we report all hyperparameters used in our experiments. Hyperparameter values in square brackets indicate ranges of values that were used for performance tuning.

\begin{center}
    \begin{tabular}{c|c} 
        \toprule
        Clip Gradient Norm & 1\\
        $\gamma_{E}$ & 0.99\\
        $\lambda$ & 0.95\\
        Learning rate & $1$x$10^{-4}$ \\
        Number of minibatches & 4\\
        Number of optimization epochs & 4\\
        Policy architecture & CNN (Mario/Atari) or MLP (Cartpole/Maze)\\
        Number of parallel actors & 2 (Cartpole/Maze) or 20 (Mario/Atari)\\
        Optimization algorithm & Adam\\
        Rollout length & 128\\
        Sticky action probability & 0.25\\
        Use Generalized Advantage Estimation & True\\
        \midrule
        Coefficient of extrinsic reward & [1, 5]\\
        Coefficient of intrinsic reward & [1, 2, 5, 10, 20, 50]\\
        $\gamma_{I}$ & 0.99\\
        Probability of terminating option & [0.5, 0.75, 0.8, 0.9, 0.95]\\
        RND output size & [2, 4, 8, 16, 32, 64, 128, 256]\\
        \bottomrule
    \end{tabular}
\end{center}


\clearpage
\section{Ablation Studies}\label{sec:ablation_studies_appendix}
\subsection*{Ablation Study 1: Adaption of ROSA to Different Controller Policies}
We claimed {\fontfamily{cmss}\selectfont Shaper} can design a reward-shaping scheme that can \textit{adapt} its shaping reward guidance of the {\fontfamily{cmss}\selectfont Controller} (to achieve the optimal policy) according to the {\fontfamily{cmss}\selectfont Controller}'s (RL) policy. 

To test this claim, we tested two versions of our agent in a corridor Maze. The maze features two goal states that are equidistant from the origin, one is a suboptimal goal with a reward of $0.5$ and the other is an optimal goal which has a reward $1$. There is also a fixed cost for each non-terminal transition. We tested this scenario with two versions of our controller: one with a standard RL Controller policy and another version in which the actions of the {\fontfamily{cmss}\selectfont Controller} are determined by a high entropy policy, we call this version of the {\fontfamily{cmss}\selectfont Controller} the \textit{high entropy controller}.\footnote{To generate this policy, we artificially increased the entropy by adjusting the temperature of a softmax function on the policy logits.}  The high entropy policy induces actions that may randomly push {\fontfamily{cmss}\selectfont Controller} towards the suboptimal goal. Therefore, in order to guide {\fontfamily{cmss}\selectfont Controller} to the optimal goal state, we expect {\fontfamily{cmss}\selectfont Shaper} to strongly shape the rewards of the {\fontfamily{cmss}\selectfont Controller} to guide {\fontfamily{cmss}\selectfont Controller} away from the suboptimal goal (and towards the optimal goal). 
\begin{figure}[ht!]
\centering
{\includegraphics[width=0.5\columnwidth]{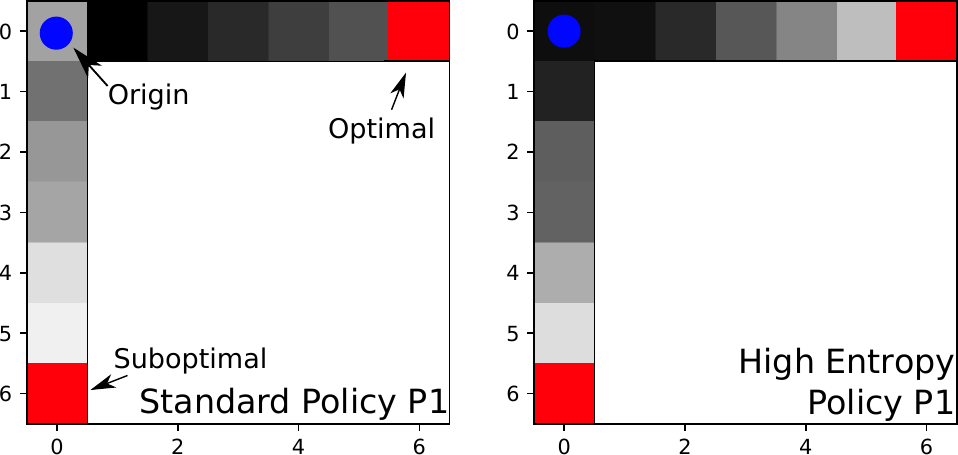}}
    \caption{Responsiveness to Controller policies}
    \label{Figure:Ablation_Studies:responsiveness}
\end{figure}
Figure \ref{Figure:Ablation_Studies:responsiveness} shows heatmaps of the added intrinsic reward (darker colours indicate higher intrinsic rewards) for the two versions of the {\fontfamily{cmss}\selectfont Controller}. With the standard policy controller, the intrinsic reward is maximal in the state to the right of the origin indicating that the {\fontfamily{cmss}\selectfont Shaper} determines that these shaping rewards are sufficient to guide {\fontfamily{cmss}\selectfont Controller} towards the optimal goal state. For the high entropy controller, the {\fontfamily{cmss}\selectfont Shaper} introduces high intrinsic rewards to the origin state as well as states beneath the origin. These rewards serve to counteract the random actions taken by the high-entropy policy that lead {\fontfamily{cmss}\selectfont Controller} towards the suboptimal goal state. It can therefore be seen that the {\fontfamily{cmss}\selectfont Shaper} adapts the shaping rewards according to the type of Controller it seeks to guide.

\subsection*{Ablation Study 2: Switching Controls}
Switching controls enable ROSA to be selective of states to which intrinsic rewards are added. This improves learnability (specifically, by reducing the computational complexity) of the learning task for the {\fontfamily{cmss}\selectfont Shaper} as there are fewer states where it must learn the optimal intrinsic reward to add to the {\fontfamily{cmss}\selectfont Controller} objective. 

To test the effect of this feature on the performance of ROSA, we compared ROSA to a modified version in which the {\fontfamily{cmss}\selectfont Shaper} must add intrinsic rewards to all states. That is, for this version of ROSA we remove the presence of the switching control mechanism for the {\fontfamily{cmss}\selectfont Shaper}. Figure \ref{Figure:Ablation_Studies} shows learning curves on the Maze environment used in the "Optimality of shaping reward" experiments in Section \ref{Section:Experiments}. As expected, the agent with the version of ROSA with switching controls learns significantly faster than the agent that uses the version of ROSA sans the switching control mechanism. For example, it takes the agent that has no switching control mechanism almost 50,000 more steps to attain an average episode return of 0.5 as compared against the agent that uses the version of our algorithm with switching controls. 

This illustrates a key benefit of switching controls which is to reduce the computational burden on {\fontfamily{cmss}\selectfont Shaper} (as it does not need to model the effects of adding intrinsic rewards in \textit{all} states) which in turn leads to both faster computation of solutions and improved performance by the {\fontfamily{cmss}\selectfont Controller}. Moreover, Maze is a relatively simple environment, expectedly the importance of the switching control is amplified in more complex environments.

Our reward-shaping method features a mechanism to selectively pick states to which intrinsic rewards are added. It also adapts its shaping rewards according to the {\fontfamily{cmss}\selectfont Controller}'s learning process. {In this section, we present the results of experiments in which we ablated each of these components.} In particular, we test the performance of ROSA in comparison to a version of ROSA with the switching mechanism removed. We then present the result of an experiment in which we investigated the ability of ROSA to adapt to different behaviour of the {\fontfamily{cmss}\selectfont Controller}. 

\begin{figure}[h!]
    \centering
            \includegraphics[width=0.5\columnwidth]{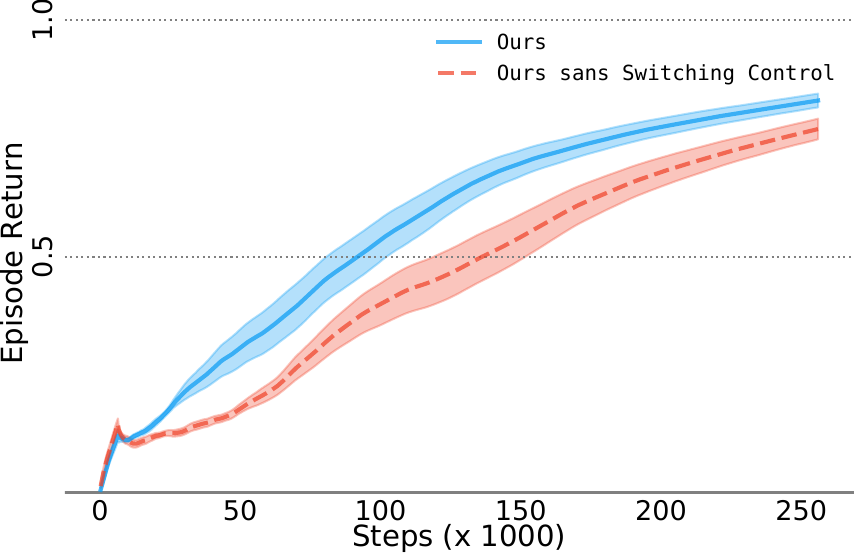}
        \caption{Ablating Switching Controls}
    \label{Figure:Ablation_Studies}
\end{figure}






\normalsize
\clearpage
\section{Notation \& Assumptions}\label{sec:notation_appendix}

We assume that $\mathcal{S}$ is defined on a probability space $(\Omega,\mathcal{F},\mathbb{P})$ and any $s\in\mathcal{S}$ is measurable with respect
to the Borel $\sigma$-algebra associated with $\mathbb{R}^p$. We denote the $\sigma$-algebra of events generated by $\{s_t\}_{t\geq 0}$
by $\mathcal{F}_t\subset \mathcal{F}$. In what follows, we denote by $\left( \mathcal{V},\|\|\right)$ any finite normed vector space and by $\mathcal{H}$ the set of all measurable functions. 

The results of the paper are built under the following assumptions which are standard within RL and stochastic approximation methods:

\textbf{Assumption 1}
The stochastic process governing the system dynamics is ergodic, that is  the process is stationary and every invariant random variable of $\{s_t\}_{t\geq 0}$ is equal to a constant with probability $1$.

\textbf{Assumption 2}
The constituent functions of the players' objectives $R$, $F$ and $L$ are in $L_2$.

\textbf{Assumption 3}
For any positive scalar $c$, there exists a scalar $\mu_c$ such that for all $s\in\mathcal{S}$ and for any $t\in\mathbb{N}$ we have: $
    \mathbb{E}\left[1+\|s_t\|^c|s_0=s\right]\leq \mu_c(1+\|s\|^c)$.

\textbf{Assumption 4}
There exists scalars $C_1$ and $c_1$ such that for any function $J$ satisfying $|J(s)|\leq C_2(1+\|s\|^{c_2})$ for some scalars $c_2$ and $C_2$ we have that: $
    \sum_{t=0}^\infty\left|\mathbb{E}\left[J(s_t)|s_0=s\right]-\mathbb{E}[J(s_0)]\right|\leq C_1C_2(1+\|s_t\|^{c_1c_2})$.

\textbf{Assumption 5}
There exists scalars $c$ and $C$ such that for any $s\in\mathcal{S}$ we have that: $
    |J(z,\cdot)|\leq C(1+\|z\|^c)$ for $J\in \{R,F,L\}$.
    
We also make the following finiteness assumption on set of switching control policies for the {\fontfamily{cmss}\selectfont Shaper}:

\textbf{Assumption 6}
For any policy $\mathfrak{g}_c$, the total number of interventions is given by $K<\infty$.


\textbf{Assumption 7}
Let $n(s)$ be the state visitation count for a given state $s\in\mathcal{S}$. For any $a\in\mathcal{A}$, the function $L(s)= 0$ for any $n(s)\geq M$ where $0<M\leq \infty$.
\clearpage
\section{Proof of Technical Results}\label{sec:proofs_appendix}

We begin the analysis with some preliminary lemmata and definitions which are useful for proving the main results.

\begin{definition}{A.1}
An operator $T: \mathcal{V}\to \mathcal{V}$ is a \textbf{contraction} with respect to a norm $\|\cdot\|$ if there exists some constant $c\in[0,1[$ such that for any $V_1,V_2\in  \mathcal{V}$ the following inequality holds:
\begin{align}
    \|TV_1-TV_2\|\leq c\|V_1-V_2\|.
\end{align}
\end{definition}

\begin{definition}{A.2}
An operator $T: \mathcal{V}\to  \mathcal{V}$ is \textbf{non-expansive} if $\forall V_1,V_2\in  \mathcal{V}$ we have:
\begin{align}
    \|TV_1-TV_2\|\leq \|V_1-V_2\|.
\end{align}
\end{definition}

\begin{lemma} \label{max_lemma}
For any
$f: \mathcal{V}\to\mathbb{R},g: \mathcal{V}\to\mathbb{R}$, we have that:
\begin{align}
\left\|\underset{a\in \mathcal{V}}{\max}\:f(a)-\underset{a\in \mathcal{V}}{\max}\: g(a)\right\| \leq \underset{a\in \mathcal{V}}{\max}\: \left\|f(a)-g(a)\right\|.    \label{lemma_1_basic_max_ineq}
\end{align}
\end{lemma}
\begin{proof}
We provide the straightforward proof of  the result given in \cite{mguni2019cutting}:
\begin{align}
f(a)&\leq \left\|f(a)-g(a)\right\|+g(a)\label{max_inequality_proof_start}
\\\implies
\underset{a\in \mathcal{V}}{\max}f(a)&\leq \underset{a\in \mathcal{V}}{\max}\{\left\|f(a)-g(a)\right\|+g(a)\}
\leq \underset{a\in \mathcal{V}}{\max}\left\|f(a)-g(a)\right\|+\underset{a\in \mathcal{V}}{\max}\;g(a). \label{max_inequality}
\end{align}
Subtracting $\underset{a\in \mathcal{V}}{\max}\;g(a)$ from both sides of (\ref{max_inequality}) gives:
\begin{align}
    \underset{a\in \mathcal{V}}{\max}f(a)-\underset{a\in \mathcal{V}}{\max}g(a)\leq \underset{a\in \mathcal{V}}{\max}\left\|f(a)-g(a)\right\|.\label{max_inequality_result_last}
\end{align}
After reversing the roles of $f$ and $g$ and performing identical steps (\ref{max_inequality_proof_start}) - (\ref{max_inequality}), we derive the desired result since the RHS of (\ref{max_inequality_result_last}) is unchanged.
\end{proof}

\begin{lemma}{A.4}\label{non_expansive_P}
The probability transition kernel $P$ is non-expansive, that is we have that:
\begin{align}
    \|PV_1-PV_2\|\leq \|V_1-V_2\|.
\end{align}
\end{lemma} 
\begin{proof}
This is a well-known result \cite{tsitsiklis1999optimal}. We state a proof using the Tonelli-Fubini theorem and the iterated law of expectations. Indeed, we observe that:
\begin{align*}
&\|PJ\|^2=\mathbb{E}\left[(PJ)^2[s_0]\right]=\mathbb{E}\left(\left[\mathbb{E}\left[J[s_1]|s_0\right]\right)^2\right]
\leq \mathbb{E}\left[\mathbb{E}\left[J^2[s_1]|s_0\right]\right] 
= \mathbb{E}\left[J^2[s_1]\right]=\|J\|^2,
\end{align*}
where we have used Jensen's inequality. This completes the proof.
\end{proof}

\section*{Proof of Theorem \ref{main_theorem}}
\begin{proof}

\subsection*{Proof of Proposition \ref{invariance_prop}}
\begin{proof}[Proof of Prop \ref{invariance_prop}]
To prove (i) of the proposition it suffices to prove that the term $\sum_{t=0}^T\gamma^{t}F(\theta_t,\theta_{t-1})I(t)$ converges to $0$ in the limit as $T\to \infty$. As in classic potential-based reward shaping \cite{ng1999policy}, central to this observation is the telescoping sum that emerges by construction of $F$.

First recall $\tilde{v}^{\pi,\pi^2}(s,I_0)$, for any $(s,I_0)\in\mathcal{S}\times\{0,1\}$ is given by:
\begin{align}
&\tilde{v}^{\pi,\pi^2}(s,I_0)=\mathbb{E}_{\pi,\pi^2}\left[\sum_{t=0}^\infty \gamma^t\left\{R(s_t,a_t)+\hat{F}(a^2_t,a^2_{t-1})I_t\right\}\right]
\\&=\mathbb{E}_{\pi,\pi^2}\left[\sum_{t=0}^\infty \gamma^tR(s_t,a_t)+\sum_{t=0}^\infty \gamma^t\hat{F}(a^2_t,a^2_{t-1})I_t\right]
\\&=\mathbb{E}_{\pi,\pi^2}\left[\sum_{t=0}^\infty \gamma^tR(s_t,a_t)\right]+\mathbb{E}_{\pi,\pi^2}\left[\sum_{t=0}^\infty \gamma^t\hat{F}(a^2_t,a^2_{t-1}))I_t\right].
\end{align}
where $I_t\equiv I(t)$ for any $t=0,1\ldots$.

Hence it suffices to prove that $\mathbb{E}_{\pi,\pi^2}\left[\sum_{t=0}^\infty \gamma^t\hat{F}(a^2_t,a^2_{t-1}))I_t\right]=0$.

Recall there a number of time steps that elapse between $\tau_k$ and $\tau_{k+1}$, now
\begin{align*}
&\sum_{t=0}^\infty\gamma^{t}\hat{F}(a^2_t,a^2_{t-1}))I(t)
\\&=\sum_{t=\tau_1+1}^{\tau_2}\gamma^{t}a^2_t-\gamma^{t-1}a^2_{t-1}+\gamma^{\tau_1}a^2_{\tau_1} +\sum_{t=\tau_3+1}^{\tau_4}\gamma^{t}a^2_t-\gamma^{t-1}a^2_{t-1}+\gamma^{\tau_3}a^2_{\tau_3}
\\&\quad+\ldots+ \sum_{t=\tau_{(2k-1)}+1}^{\tau_{2k
}}\gamma^{t}a^2_t-\gamma^{t-1}a^2_{t-1}+\gamma^{\tau_1}a^2_{\tau_{2k+1}}+\ldots+
\\&=\sum_{t=\tau_1}^{\tau_2-1}\gamma^{t+1}a^2_{t+1}-\gamma^{t}a^2_{t}+\gamma^{\tau_1}a^2_{\tau_1}+\sum_{t=\tau_3}^{\tau_4-1}\gamma^{t+1}a^2_{t+1}-\gamma^{t}a^2_{t}+\gamma^{\tau_3}a^2_{\tau_3}
\\&\quad+\ldots+ \sum_{t=\tau_{(2k-1)}}^{\tau_{2K-1}}\gamma^{t}a^2_t-\gamma^{t-1}a^2_{t-1}+\gamma^{\tau_{2k-1}}a^2_{\tau_{2k-1}}+\ldots+
\\&=\sum_{k=1}^\infty\sum_{t=\tau_{2k-1}}^{\tau_{2K-1}}\gamma^{t+1}a^2_{t+1}-\gamma^{t}a^2_{t}-\sum_{k=1}^\infty\gamma^{\tau_{2k-1}}a^2_{\tau_{2k-1}}
\\&=\sum_{k=1}^\infty\gamma^{\tau_{2k}}a^2_{\tau_{2k}}-\sum_{k=1}^\infty\gamma^{\tau_{2k-1}}a^2_{\tau_{2k-1}}
\\&=\sum_{k=1}^\infty\gamma^{\tau_{2k}}0-\sum_{k=1}^\infty\gamma^{\tau_{2k-1}}0=0,
\end{align*}
where we have used the fact that by construction  $a^2_t\equiv 0$ whenever $t=\tau_1,\tau_2,\ldots$.


In what follows, for any state $s\in\cS$, we denote by $L_n(s)$ the value of $L(s)$ when the state $s$ has been visited $n=0,1,\ldots$ times.

We now note that it is easy to see that $v^{\pi,\pi^2}_2(s_0,I_0)$ is bounded above, indeed using the above we have that

\begin{align}
v^{\pi,\pi^2}_2(s_0,I_0)&=\mathbb{E}_{\pi,\pi^2}\left[ \sum_{t=0}^\infty \gamma^t\left(\hat{R} -\sum_{k\geq 1} \delta^t_{\tau_{2k-1}}
+L_n(s_t)\right)\right]
\\&=\mathbb{E}_{\pi,\pi^2}\left[ \sum_{t=0}^\infty \gamma^t\left({R} -\sum_{k\geq 1} \delta^t_{\tau_{2k-1}}
+L_n(s_t)\right)+\sum_{t=0}^\infty \gamma^t\hat{F}I_t\right]
\\&\leq \mathbb{E}_{\pi,\pi^2}\left[ \sum_{t=0}^\infty \gamma^t\left({R} +L_n(s_t)\right)\right]
\\&\leq \left|\mathbb{E}_{\pi,\pi^2}\left[ \sum_{t=0}^\infty \gamma^t\left({R}+L_n(s_t)\right)\right]\right|
\\&\leq \mathbb{E}_{\pi,\pi^2}\left[ \sum_{t=0}^\infty \gamma^t\left\|{R} +
L_n\right\|\right]
\\&\leq  \sum_{t=0}^\infty \gamma^t\left(\left\|{R}\right\| +\left\|
L_n\right\|\right)
\\&=\frac{1}{1-\gamma}\left(\left\|{R}\right\| +\left\|
L\right\|\right),
\end{align}
using the triangle inequality, the definition of $\hat{R}$ and the (upper-)boundedness of $L$ and $R$ (Assumption 5).
We now note that by the dominated convergence theorem we have that $\forall (s_0,I_0)\in\mathcal{S}\times\{0,1\}$ that
\begin{align}
&\underset{n\to\infty}{\lim}\; v^{\pi,\pi^2}_2(s_0,I_0)  = \underset{n\to\infty}{\lim}\; \mathbb{E}_{\pi,\pi^2}\left[ \sum_{t=0}^\infty \gamma^t\left(\hat{R} -\sum_{k\geq 1} \delta^t_{\tau_{2k-1}}
+L_n(s_t)\right)\right]
\\&=\mathbb{E}_{\pi,\pi^2}\underset{n\to\infty}{\lim}\;\left[ \sum_{t=0}^\infty \gamma^t\left(\hat{R} -\sum_{k\geq 1} \delta^t_{\tau_{2k-1}}
+L_n(s_t)\right)\right]
\\&=\mathbb{E}_{\pi,\pi^2}\left[ \sum_{t=0}^\infty \gamma^t\left(\hat{R} -\sum_{k\geq 1} \delta^t_{\tau_{2k-1}}\right)\right]
\\&=\mathbb{E}_{\pi,\pi^2}\left[ \sum_{t=0}^\infty \gamma^t\left(R -\sum_{k\geq 1} \delta^t_{\tau_{2k-1}}\right)\right]=-\frac{K}{1-\gamma}+v^{\pi}(s_0),
\end{align}
using Assumption 6 in the last step, after which we deduce (i).

To deduce (ii) we simply note that $v^{\pi,\pi^2}_2(s_0,I_0)$ and $v^{\pi}(s_0)$ differ by only a constant and hence share the same optimisation.

\end{proof}




\section*{Proof of Theorem \ref{theorem:existence}}
\begin{proof}
Theorem \ref{theorem:existence} is proved by firstly showing that when the players jointly maximise the same objective there exists a fixed point equilibrium of the game when all players use Markov policies and {\fontfamily{cmss}\selectfont Shaper} uses switching control. The proof then proceeds by showing that the MG $\mathcal{G}$ admits a dual representation as an MG in which jointly maximise the same objective which has a stable point that can be computed by solving an MDP. Thereafter, we use both results to prove the existence of a fixed point for the game as a limit point of a sequence generated by successively applying the Bellman operator to a test function.  

Therefore, the scheme of the proof is summarised with the following steps:
\begin{itemize}
    \item[\textbf{I)}] Prove that the solution to Markov Team games (that is games in which both players maximise \textit{identical objectives}) in which one of the players uses switching control is the limit point of a sequence of Bellman operators (acting on some test function).
    \item[\textbf{II)}] Prove that for the MG $\mathcal{G}$ that is there exists a function $B^{\pi,\pi^2}:\mathcal{S}\times \{0,1\}\to \mathbb{R}$ such that\footnote{This property is analogous to  the condition in Markov potential games \cite{macua2018learning,mguni2021learning}} $
 v^{\pi,\pi^2}_i(z)-v^{{\pi'},\pi^2}_i(z)
=B^{\pi,\pi^2}(z)-B^{{\pi'},\pi^2}(z),\;\;\forall z\equiv (s,I_0)\in\mathcal{S}\times \{0,1\}, \forall i\in\{1,2\}$.
    \item[\textbf{III)}] Prove that the MG $\mathcal{G}$ has a dual representation as a \textit{Markov Team Game} which admits a representation as an MDP.
\end{itemize}

\subsection*{Proof of Part \textbf{I}}
We begin by defining some objects which are central to the analysis. 
For any $\pi\in\Pi$ and $\pi^2\in\Pi^2$, given a function $V^{\pi,\pi^2}:\mathcal{S}\times\mathbb{N}\to\mathbb{R}$, we define the \textit{intervention operator} $\mathcal{M}^{\pi,\pi^2}$ by 
\begin{align}
\mathcal{M}^{\pi,\pi^2}V^{\pi,\pi^2}(s_{\tau_k},I(\tau_k)):=\hat{R}_1(s_{\tau_k},I(\tau_k),a_{\tau_k},a^2_{\tau_k},\cdot)-1+\gamma\sum_{s'\in\mathcal{S}}P(s';a_{\tau_k},s)V^{\pi,\pi^2}(s',I(\tau_{k+1}))
\end{align}
for any $s_{\tau_k}\in\mathcal{S}$ and $\forall \tau_k$ where $a_{\tau_k}\sim \pi(\cdot|s_{\tau_k})$ and where $a^2_{\tau_k}\sim  \pi^2(\cdot|s_{\tau_k})$. 

We define the Bellman operator $T$ of the game $\mathcal{G}$  by 
\begin{align}
T V^{\pi,\pi^2}(s_{\tau_k},I(\tau_k)):=\max\Big\{\mathcal{M}^{\pi,\pi^2}V(s_{\tau_k},I(\tau_k)),  R(s_{\tau_k},a)+\gamma\underset{a\in\mathcal{A}}{\max}\;\sum_{s'\in\mathcal{S}}P(s';a,s_{\tau_k})V(s',I(\tau_k))\Big\}. 
\end{align}

Our first result proves that the operator  $T$ is a contraction operator. First let us recall that the \textit{switching time} $\tau_k$ is defined recursively $\tau_k=\inf\{t>\tau_{k-1}|s_t\in A,\tau_k\in\mathcal{F}_t\}$ where $A=\{s\in \mathcal{S},m\in M|\mathfrak{g}_2(m|s_t)>0\}$.
To this end, we show that the following bounds holds:
\begin{lemma}\label{lemma:bellman_contraction}
The Bellman operator $T$ is a contraction, that is the following bound holds:
\begin{align*}
&\left\|Tv-Tv'\right\|\leq \gamma\left\|v-v'\right\|.
\end{align*}
for any $v\in L_2$.
\end{lemma}

\begin{proof}
Recall we define the Bellman operator $T$ of $\mathcal{G}$ acting on a function $v:\mathcal{S}\times\mathbb{N}\to\mathbb{R}$ by
\begin{align}
T v(s_{\tau_k},I(\tau_k)):=\max\left\{\mathcal{M}^{\pi,\pi^2}v(s_{\tau_k},I(\tau_k)),\left[ R(s_{\tau_k},{a})+\gamma\underset{{a}\in{\mathcal{A}}}{\max}\;\sum_{s'\in\mathcal{S}}P(s';{a},s_{\tau_k})v(s',I(\tau_k))\right]\right\}\label{bellman_proof_start}
\end{align}

In what follows and for the remainder of the script, we employ the following shorthands:
\begin{align*}
&\mathcal{P}^{{a}}_{ss'}=:\sum_{s'\in\mathcal{S}}P(s';{a},s), \quad\mathcal{P}^{{\pi}}_{ss'}=:\sum_{{a}\in{\mathcal{A}}}{\pi}({a}|s)\mathcal{P}^{{a}}_{ss'}, \quad \mathcal{R}^{{\pi}}(z_{t}):=\sum_{{a}_t\in{\mathcal{A}}}{\pi}({a}_t|s)\hat{R}(z_t,{a}_t,\theta_t,\theta_{t-1})
\end{align*}

To prove that $T$ is a contraction, we consider the three cases produced by \eqref{bellman_proof_start}, that is to say we prove the following statements:

i) $\qquad\qquad
\left| \Theta(z_t,{a},a^2_t,a^2_{t-1})+\gamma\underset{{a}\in{\mathcal{A}}}{\max}\;\mathcal{P}^{{a}}_{s's_t}v(s',\cdot)-\left( \Theta(z_t,{a},a^2_t,a^2_{t-1})+\gamma\underset{{a}\in{\mathcal{A}}}{\max}\;\mathcal{P}^{{a}}_{s's_t}v'(s',\cdot)\right)\right|\leq \gamma\left\|v-v'\right\|$

ii) $\qquad\qquad
\left\|\mathcal{M}^{\pi,\pi^2}v-\mathcal{M}^{\pi,\pi^2}v'\right\|\leq    \gamma\left\|v-v'\right\|,\qquad \qquad$
  (and hence $\mathcal{M}$ is a contraction).

iii) $\qquad\qquad
    \left\|\mathcal{M}^{\pi,\pi^2}v-\left[ \Theta(\cdot,{a})+\gamma\underset{{a}\in{\mathcal{A}}}{\max}\;\mathcal{P}^{{a}}v'\right]\right\|\leq \gamma\left\|v-v'\right\|.
$
where $z_t\equiv (s_t,I_t)\in\mathcal{S}\times \{0,1\}$.

We begin by proving i).

Indeed, for any ${a}\in{\mathcal{A}}$ and $\forall z_t\in\mathcal{S}\times\{0,1\}, \forall \theta_t,\theta_{t-1}\in \Theta, \forall s'\in\mathcal{S}$ we have that 
\begin{align*}
&\left| \Theta(z_t,{a},a^2_t,a^2_{t-1})+\gamma\mathcal{P}^\pi_{s's_t}v(s',\cdot)-\left[ \Theta(z_t,{a},a^2_t,a^2_{t-1})+\gamma\underset{{a}\in{\mathcal{A}}}{\max}\;\;\mathcal{P}^{{a}}_{s's_t}v'(s',\cdot)\right]\right|
\\&\leq \underset{{a}\in{\mathcal{A}}}{\max}\;\left|\gamma\mathcal{P}^{{a}}_{s's_t}v(s',\cdot)-\gamma\mathcal{P}^{{a}}_{s's_t}v'(s',\cdot)\right|
\\&\leq \gamma\left\|Pv-Pv'\right\|
\\&\leq \gamma\left\|v-v'\right\|,
\end{align*}
again using the fact that $P$ is non-expansive and Lemma \ref{max_lemma}.

We now prove ii).

For any $\tau\in\mathcal{F}$, define by $\tau'=\inf\{t>\tau|s_t\in A,\tau\in\mathcal{F}_t\}$. Now using the definition of $\mathcal{M}$ we have that for any $s_\tau\in\mathcal{S}$
\begin{align*}
&\left|(\mathcal{M}^{\pi,\pi^2}v-\mathcal{M}^{\pi,\pi^2}v')(s_{\tau},I(\tau))\right|
\\&\leq \underset{{a}_\tau,a^2_\tau,a^2_{\tau-1}\in {\mathcal{A}}\times \Theta^2}{\max}    \Bigg|\Theta(z_\tau,{a}_\tau,a^2_\tau,a^2_{\tau-1})-1+\gamma\mathcal{P}^{{\pi}}_{s's_\tau}\mathcal{P}^{{a}}v(s_{\tau},I(\tau'))
\\&\qquad\qquad-\left(\Theta(z_\tau,{a}_\tau,a^2_\tau,a^2_{\tau-1})-1+\gamma\mathcal{P}^{{\pi}}_{s's_\tau}\mathcal{P}^{{a}}v'(s_{\tau},I(\tau'))\right)\Bigg| 
\\&= \gamma\left|\mathcal{P}^{{\pi}}_{s's_\tau}\mathcal{P}^{{a}}v(s_{\tau},I(\tau'))-\mathcal{P}^{{\pi}}_{s's_\tau}\mathcal{P}^{{a}}v'(s_{\tau},I(\tau'))\right| 
\\&\leq \gamma\left\|Pv-Pv'\right\|
\\&\leq \gamma\left\|v-v'\right\|,
\end{align*}
using the fact that $P$ is non-expansive. The result can then be deduced easily by applying max on both sides.

We now prove iii). We split the proof of the statement into two cases:

\textbf{Case 1:} 
\begin{align}\mathcal{M}^{\pi,\pi^2}v(s_{\tau},I(\tau))-\left(\Theta(z_\tau,{a}_\tau,a^2_{\tau},a^2_{\tau-1})+\gamma\underset{{a}\in{\mathcal{A}}}{\max}\;\mathcal{P}^{{a}}_{s's_\tau}v'(s',I(\tau))\right)<0.
\end{align}

We now observe the following:
\begin{align*}
&\mathcal{M}^{\pi,\pi^2}v(s_{\tau},I(\tau))-\Theta(z_\tau,{a}_\tau,a^2_{\tau},a^2_{\tau-1})+\gamma\underset{{a}\in{\mathcal{A}}}{\max}\;\mathcal{P}^{{a}}_{s's_\tau}v'(s',I(\tau))
\\&\leq\max\left\{\Theta(z_\tau,{a}_\tau,a^2_{\tau},a^2_{\tau-1})+\gamma\mathcal{P}^{{\pi}}_{s's_\tau}\mathcal{P}^{{a}}v(s',I({\tau})),\mathcal{M}^{\pi,\pi^2}v(s_{\tau},I(\tau))\right\}
\\&\qquad-\Theta(z_\tau,{a}_\tau,a^2_{\tau},a^2_{\tau-1})+\gamma\underset{{a}\in{\mathcal{A}}}{\max}\;\mathcal{P}^{{a}}_{s's_\tau}v'(s',I(\tau))
\\&\leq \Bigg|\max\left\{\Theta(z_\tau,{a}_\tau,a^2_{\tau},a^2_{\tau-1})+\gamma\mathcal{P}^{{\pi}}_{s's_\tau}\mathcal{P}^{{a}}v(s',I({\tau})),\mathcal{M}^{\pi,\pi^2}v(s_{\tau},I(\tau))\right\}
\\&\qquad-\max\left\{\Theta(z_\tau,{a}_\tau,a^2_{\tau},a^2_{\tau-1})+\gamma\underset{{a}\in{\mathcal{A}}}{\max}\;\mathcal{P}^{{a}}_{s's_\tau}v'(s',I({\tau})),\mathcal{M}^{\pi,\pi^2}v(s_{\tau},I(\tau))\right\}
\\&+\max\left\{\Theta(z_\tau,{a}_\tau,a^2_{\tau},a^2_{\tau-1})+\gamma\underset{{a}\in{\mathcal{A}}}{\max}\;\mathcal{P}^{{a}}_{s's_\tau}v'(s',I({\tau})),\mathcal{M}^{\pi,\pi^2}v(s_{\tau},I(\tau))\right\}
\\&\qquad-\Theta(z_\tau,{a}_\tau,a^2_{\tau},a^2_{\tau-1})+\gamma\underset{{a}\in{\mathcal{A}}}{\max}\;\mathcal{P}^{{a}}_{s's_\tau}v'(s',I(\tau))\Bigg|
\\&\leq \Bigg|\max\left\{\Theta(z_\tau,{a}_\tau,a^2_{\tau},a^2_{\tau-1})+\gamma\underset{{a}\in{\mathcal{A}}}{\max}\;\mathcal{P}^{{a}}_{s's_\tau}v(s',I({\tau})),\mathcal{M}^{\pi,\pi^2}v(s_{\tau},I(\tau))\right\}
\\&\qquad-\max\left\{\Theta(z_\tau,{a}_\tau,a^2_{\tau},a^2_{\tau-1})+\gamma\underset{{a}\in{\mathcal{A}}}{\max}\;\mathcal{P}^{{a}}_{s's_\tau}v'(s',I({\tau})),\mathcal{M}^{\pi,\pi^2}v(s_{\tau},I(\tau))\right\}\Bigg|
\\&\qquad+\Bigg|\max\left\{\Theta(z_\tau,{a}_\tau,a^2_{\tau},a^2_{\tau-1})+\gamma\underset{{a}\in{\mathcal{A}}}{\max}\;\mathcal{P}^{{a}}_{s's_\tau}v'(s',I({\tau})),\mathcal{M}^{\pi,\pi^2}v(s_{\tau},I(\tau))\right\}\\&\qquad\qquad-\Theta(z_\tau,{a}_\tau,a^2_{\tau},a^2_{\tau-1})+\gamma\underset{{a}\in{\mathcal{A}}}{\max}\;\mathcal{P}^{{a}}_{s's_\tau}v'(s',I(\tau))\Bigg|
\\&\leq \gamma\underset{a\in\mathcal{A}}{\max}\;\left|\mathcal{P}^{{\pi}}_{s's_\tau}\mathcal{P}^{{a}}v(s',I(\tau))-\mathcal{P}^{{\pi}}_{s's_\tau}\mathcal{P}^{{a}}v'(s',I(\tau))\right|
\\&\qquad+\left|\max\left\{0,\mathcal{M}^{\pi,\pi^2}v(s_{\tau},I(\tau))-\left(\Theta(z_\tau,{a}_\tau,a^2_{\tau},a^2_{\tau-1})+\gamma\underset{{a}\in{\mathcal{A}}}{\max}\;\mathcal{P}^{{a}}_{s's_\tau}v'(s',I(\tau))\right)\right\}\right|
\\&\leq \gamma\left\|Pv-Pv'\right\|
\\&\leq \gamma\|v-v'\|,
\end{align*}
where we have used the fact that for any scalars $a,b,c$ we have that $
    \left|\max\{a,b\}-\max\{b,c\}\right|\leq \left|a-c\right|$ and the non-expansiveness of $P$.

\textbf{Case 2: }
\begin{align*}\mathcal{M}^{\pi,\pi^2}v(s_{\tau},I(\tau))-\left(\Theta(z_\tau,{a}_\tau,a^2_{\tau},a^2_{\tau-1})+\gamma\underset{{a}\in{\mathcal{A}}}{\max}\;\mathcal{P}^{{a}}_{s's_\tau}v'(s',I(\tau))\right)\geq 0.
\end{align*}

\begin{align*}
&\mathcal{M}^{\pi,\pi^2}v(s_{\tau},I(\tau))-\left(\Theta(z_\tau,{a}_\tau,a^2_{\tau},a^2_{\tau-1})+\gamma\underset{{a}\in{\mathcal{A}}}{\max}\;\mathcal{P}^{{a}}_{s's_\tau}v'(s',I(\tau))\right)
\\&\leq \mathcal{M}^{\pi,\pi^2}v(s_{\tau},I(\tau))-\left(\Theta(z_\tau,{a}_\tau,a^2_{\tau},a^2_{\tau-1})+\gamma\underset{{a}\in{\mathcal{A}}}{\max}\;\mathcal{P}^{{a}}_{s's_\tau}v'(s',I(\tau))\right)+1
\\&\leq \Theta(z_\tau,{a}_\tau,a^2_{\tau},a^2_{\tau-1})-1+\gamma\mathcal{P}^{{\pi}}_{s's_\tau}\mathcal{P}^{{a}}v(s',I(\tau'))
\\&\qquad\qquad\qquad\qquad\quad-\left(\Theta(z_\tau,{a}_\tau,a^2_{\tau},a^2_{\tau-1})-1+\gamma\underset{{a}\in{\mathcal{A}}}{\max}\;\mathcal{P}^{{a}}_{s's_\tau}v'(s',I(\tau))\right)
\\&\leq \gamma\underset{{a}\in{\mathcal{A}}}{\max}\;\left|\mathcal{P}^{{\pi}}_{s's_\tau}\mathcal{P}^{{a}}\left(v(s',I(\tau'))-v'(s',I(\tau))\right)\right|
\\&\leq \gamma\left|v(s',I(\tau'))-v'(s',I(\tau))\right|
\\&\leq \gamma\left\|v-v'\right\|,
\end{align*}
again using the fact that $P$ is non-expansive. Hence we have succeeded in showing that for any $v\in L_2$ we have that
\begin{align}
    \left\|\mathcal{M}^{\pi,\pi^2}v-\underset{{a}\in{\mathcal{A}}}{\max}\;\left[ R(\cdot,a)+\gamma\mathcal{P}^{{a}}v'\right]\right\|\leq \gamma\left\|v-v'\right\|.\label{off_M_bound_gen}
\end{align}
Gathering the results of the three cases gives the desired result. 
\end{proof}
\subsection*{Proof of Part \textbf{II}}

To prove Part \textbf{II}, we prove the following result:
\begin{proposition}\label{dpg_proposition}
For any ${\pi}\in{\Pi}$ and for any {\fontfamily{cmss}\selectfont Shaper} policy $\pi^2$, there exists a function $B^{\pi,\pi^2}:\mathcal{S}\times \{0,1\}\to \mathbb{R}$ such that
\begin{align}
 v^{\pi,\pi^2}_i(z)-v^{{\pi'},\pi^2}_i(z)
=B^{\pi,\pi^2}(z)-B^{{\pi'},\pi^2}(z),\;\;\forall z\equiv (s,I_0)\in\mathcal{S}\times \{0,1\}\label{potential_relation_proof}
\end{align}
where in particular the function $B$ is given by:
\begin{align}
B^{\pi,\pi^2}(s_0,I_0) =\mathbb{E}_{\pi,\pi^2}\left[ \sum_{t=0}^\infty \gamma^tR \right],\end{align}
for any $(s_0,I_0)\in\mathcal{S}\times \{0,1\}$.
\end{proposition}
\begin{proof}
Note that by the deduction of (ii) in Prop \ref{invariance_prop}, we may consider the following quantity for the {\fontfamily{cmss}\selectfont Shaper} expected return:
\begin{align}
    \hat{v}^{\pi,\pi^2}_2(s_0,I_0)=\mathbb{E}_{\pi,\pi^2}\left[ \sum_{t=0}^\infty \gamma^t\left(R -\sum_{k\geq 1} \delta^t_{\tau_{2k-1}}\right)\right].
\end{align}

Therefore, we immediately observe that 
\begin{align}
    \hat{v}^{\pi,\pi^2}_2(s_0,I_0)=B^{\pi,\pi^2}(s_0,I_0)-K, \;\; \forall (s_0,I_0)\in\mathcal{S}\times\{0,1\}.
\end{align}
We therefore immediately deduce that for any two {\fontfamily{cmss}\selectfont Shaper} policies $\pi^2$ and $\pi'^2$ the following expression holds $\forall (s_0,I_0)\in\mathcal{S}\times\{0,1\}$:
\begin{align}
    \hat{v}^{\pi,\pi^2}_2(s_0,I_0)-\hat{v}^{{\pi},\pi'^2}_2(s_0,I_0)=B^{\pi,\pi^2}(s_0,I_0)-B^{{\pi},\pi'^2}(s_0,I_0).
\end{align}

Our aim now is to show that the following expression holds $\forall (s_0,I_0)\in\mathcal{S}\times\{0,1\}$:
\begin{align}\nonumber
 v^{\pi,\pi^2}(I_{0},s_{0})-v^{{\pi'},\pi^2}(I_{0},s_{0})=B^{\pi,\pi^2}(I_{0},s_{0})-B^{{\pi'},\pi^2}(I_{0},s_{0}),\;\; \forall i \in\mathcal{N}
\end{align}
This is manifest from the construction of $B$.
\end{proof}

\subsection*{Proof of Part \textbf{III}}

We begin by recalling that a \textit{Markov strategy} is a policy $\pi^i: \mathcal{S} \times \mathcal{A}_i \rightarrow [0,1]$ which requires as input only the current system state (and not the game history or the other player's action or strategy \cite{mguni2018viscosity}). With this, we give a formal description of the stable points of $\mathcal{G}$ in Markov strategies. 
\begin{definition}
A policy profile $\boldsymbol{\hat{\pi}}=(\hat{\pi}^1,\hat{\pi}^2)\in\boldsymbol{\Pi}$ is a Markov perfect equilibrium (MPE) if the following holds $\forall i\neq j\in\{1,2\}, \;\forall \hat{\pi}'\in\Pi_i$: $
v_i^{(\hat{\pi}^i,\hat{\pi}^{j})}(s_0,I_0)\geq v_i^{(\hat{\pi}',\hat{\pi}^{j})}(s_0,I_0),\forall (s_0,I_0)\in \mathcal{S}\times \{0,1\}$.
\end{definition}%
The MPE describes a configuration in policies in which no player can increase their payoff by changing (unilaterally) their policy. Crucially, it defines the stable points to which independent learners converge (if they converge at all). 

\begin{proposition}\label{reduction_prop}
The following implication holds: 
\begin{align}
\boldsymbol{\sigma}\in \underset{{g',\boldsymbol{\pi'}}\in\boldsymbol{\Pi}}{\arg\sup}\; B^{g',{\boldsymbol{\pi'}}}(s)\implies \boldsymbol{\sigma}\in NE\{\mathcal{G}\}.
\end{align}
where $B$ is the function in Prop. \ref{dpg_proposition}.
\end{proposition}
Prop. \ref{reduction_prop} indicates that the game has an equivalent representation in which all agents maximise the same function and thus  play a \textit{team game}.
\begin{proof}
We do the proof by contradiction. Let $\boldsymbol{\sigma}=(\pi,\pi^2,g)\in \underset{\pi'\in\Pi,\pi^2\in \Pi^2, g'}{\arg\sup}\; B^{\pi',\pi;^2,g'}(s)$ for any $s\in\mathcal{S}$. Let us now therefore assume that $\boldsymbol{\sigma}\notin NE\{\mathcal{G}\}$, hence there exists some other policy profile $\boldsymbol{\tilde{\sigma}}=(\tilde{\pi},g)$ which contains at least one profitable deviation in policy by the {\fontfamily{cmss}\selectfont Controller} so that $\pi'\neq \pi$ and  $v^{\pi',\pi^2,g}(s)> v^{\pi,\pi^2,g}(s)$ (using the preservation of signs of integration). Prop. \ref{dpg_proposition} however implies that $B^{\pi',\pi^2,g}(s)-B^{\pi,\pi^2,g}(s)>0$ which is a contradiction since $\boldsymbol{\sigma}=(\pi,\pi^2,g)$ is a maximum of $B$.  The proof can be straightforwardly adapted to cover the case in which the deviating agent is the {\fontfamily{cmss}\selectfont Shaper} after which we deduce the desired result.
\end{proof}
\end{proof}

To prove part ii), we make use of the following result:
\begin{theorem}[Theorem 1, pg 4 in \cite{jaakkola1994convergence}]
Let $\Xi_t(s)$ be a random process that takes values in $\mathbb{R}^n$ and given by the following:
\begin{align}
    \Xi_{t+1}(s)=\left(1-\alpha_t(s)\right)\Xi_{t}(s)\alpha_t(s)L_t(s),
\end{align}
then $\Xi_t(s)$ converges to $0$ with probability $1$ under the following conditions:
\begin{itemize}
\item[i)] $0\leq \alpha_t\leq 1, \sum_t\alpha_t=\infty$ and $\sum_t\alpha_t<\infty$
\item[ii)] $\|\mathbb{E}[L_t|\mathcal{F}_t]\|\leq \gamma \|\Xi_t\|$, with $\gamma <1$;
\item[iii)] ${\rm Var}\left[L_t|\mathcal{F}_t\right]\leq c(1+\|\Xi_t\|^2)$ for some $c>0$.
\end{itemize}
\end{theorem}
\begin{proof}
To prove the result, we show (i) - (iii) hold. Condition (i) holds by choice of learning rate. It therefore remains to prove (ii) - (iii). We first prove (ii). For this, we consider our variant of the Q-learning update rule:
\begin{align*}
Q_{t+1}(s_t,I_t,a_t)=Q_{t}&(s_t,I_t,a_t)
\\&\begin{aligned}
+\alpha_t(s_t,I_t,a_t)\Bigg[\max\left\{\mathcal{M}^{\pi,\pi^2}Q(s_{\tau_k},I_{\tau_k},a), \phi(s_{\tau_k},a)+\gamma\underset{a'\in\mathcal{A}}{\max}\;Q(s',I_{\tau_k},a')\right\}&
\\-Q_{t}(s_t,I_t,a_t)\Bigg]&.
\end{aligned}
\end{align*}
After subtracting $Q^\star(s_t,I_t,a_t)$ from both sides and some manipulation we obtain that:
\begin{align*}
&\Xi_{t+1}(s_t,I_t,a_t)
\\&=(1-\alpha_t(s_t,I_t,a_t))\Xi_{t}(s_t,I_t,a_t)
\\&\begin{aligned}
\qquad\qquad\qquad\qquad\;\;+\alpha_t(s_t,I_t,a_t))\Bigg[\max\left\{\mathcal{M}^{\pi,\pi^2}Q(s_{\tau_k},I_{\tau_k},a), \phi(s_{\tau_k},a)+\gamma\underset{a'\in\mathcal{A}}{\max}\;Q(s',I_{\tau_k},a')\right\}&
\\-Q^\star(s_t,I_t,a_t)\Bigg]&,
\end{aligned}  
\end{align*}
where $\Xi_{t}(s_t,I_t,a_t):=Q_t(s_t,I_t,a_t)-Q^\star(s_t,I_t,a_t)$.

Let us now define by 
\begin{align*}
L_t(s_{\tau_k},I_{\tau_k},a):=\max\left\{\mathcal{M}^{\pi,\pi^2}Q(s_{\tau_k},I_{\tau_k},a), \phi(s_{\tau_k},a)+\gamma\underset{a'\in\mathcal{A}}{\max}\;Q(s',I_{\tau_k},a')\right\}-Q^\star(s_t,I_t,a).
\end{align*}
Then
\begin{align}
\Xi_{t+1}(s_t,I_t,a_t)=(1-\alpha_t(s_t,I_t,a_t))\Xi_{t}(s_t,I_t,a_t)+\alpha_t(s_t,I_t,a_t))\left[L_t(s_{\tau_k},a)\right].   
\end{align}

We now observe that
\begin{align}\nonumber
&\mathbb{E}\left[L_t(s_{\tau_k},I_{\tau_k},a)|\mathcal{F}_t\right]
\\&\nonumber=\sum_{s'\in\mathcal{S}}P(s';a,s_{\tau_k})\max\left\{\mathcal{M}^{\pi,\pi^2}Q(s_{\tau_k},I_{\tau_k},a), \phi(s_{\tau_k},a)+\gamma\underset{a'\in\mathcal{A}}{\max}\;Q(s',I_{\tau_k},a')\right\}-Q^\star(s_{\tau_k},a)
\\&= T_\phi Q_t(s,I_{\tau_k},a)-Q^\star(s,I_{\tau_k},a). \label{expectation_L}
\end{align}
Now, using the fixed point property that implies $Q^\star=T_\phi Q^\star$, we find that
\begin{align}\nonumber
    \mathbb{E}\left[L_t(s_{\tau_k},I_{\tau_k},a)|\mathcal{F}_t\right]&=T_\phi Q_t(s,I_{\tau_k},a)-T_\phi Q^\star(s,I_{\tau_k},a)
    \\&\leq\left\|T_\phi Q_t-T_\phi Q^\star\right\|\nonumber
    \\&\leq \gamma\left\| Q_t- Q^\star\right\|_\infty=\gamma\left\|\Xi_t\right\|_\infty.
\end{align}
using the contraction property of $T$ established in Lemma \ref{lemma:bellman_contraction}. This proves (ii).

We now prove iii), that is
\begin{align}
    {\rm Var}\left[L_t|\mathcal{F}_t\right]\leq c(1+\|\Xi_t\|^2).
\end{align}
Now by \eqref{expectation_L} we have that
\begin{align*}
  {\rm Var}\left[L_t|\mathcal{F}_t\right]&= {\rm Var}\left[\max\left\{\mathcal{M}^{\pi,\pi^2}Q(s_{\tau_k},I_{\tau_k},a), \phi(s_{\tau_k},a)+\gamma\underset{a'\in\mathcal{A}}{\max}\;Q(s',I_{\tau_k},a')\right\}-Q^\star(s_t,I_t,a)\right]
  \\&= \mathbb{E}\Bigg[\Bigg(\max\left\{\mathcal{M}^{\pi,\pi^2}Q(s_{\tau_k},I_{\tau_k},a), \phi(s_{\tau_k},a)+\gamma\underset{a'\in\mathcal{A}}{\max}\;Q(s',I_{\tau_k},a')\right\}
  \\&\qquad\qquad\qquad\qquad\qquad\quad\quad\quad-Q^\star(s_t,I_t,a)-\left(T_\Phi Q_t(s,I_{\tau_k},a)-Q^\star(s,I_{\tau_k},a)\right)\Bigg)^2\Bigg]
      \\&= \mathbb{E}\left[\left(\max\left\{\mathcal{M}^{\pi,\pi^2}Q(s_{\tau_k},I_{\tau_k},a), \phi(s_{\tau_k},a)+\gamma\underset{a'\in\mathcal{A}}{\max}\;Q(s',I_{\tau_k},a')\right\}-T_\Phi Q_t(s,I_{\tau_k},a)\right)^2\right]
    \\&= {\rm Var}\left[\max\left\{\mathcal{M}^{\pi,\pi^2}Q(s_{\tau_k},I_{\tau_k},a), \phi(s_{\tau_k},a)+\gamma\underset{a'\in\mathcal{A}}{\max}\;Q(s',I_{\tau_k},a')\right\}-T_\Phi Q_t(s,I_{\tau_k},a))^2\right]
    \\&\leq c(1+\|\Xi_t\|^2),
\end{align*}
for some $c>0$ where the last line follows due to the boundedness of $Q$ (which follows from Assumptions 2 and 4). This concludes the proof of the Theorem.
\end{proof}
\clearpage
\section*{Proof of Convergence with Linear Function Approximation}
First let us recall the statement of the theorem:
\begin{customthm}{3}
ROSA converges to a limit point $r^\star$ which is the unique solution to the equation:
\begin{align}
\Pi \mathfrak{F} (\Phi r^\star)=\Phi r^\star, \qquad \text{a.e.}
\end{align}
where we recall that for any test function $v \in \mathcal{V}$, the operator $\mathfrak{F}$ is defined by $
    \mathfrak{F}v:=\Theta+\gamma P \max\{\mathcal{M}v,v\}$.

Moreover, $r^\star$ satisfies the following:
\begin{align}
    \left\|\Phi r^\star - Q^\star\right\|\leq c\left\|\Pi Q^\star-Q^\star\right\|.
\end{align}
\end{customthm}

The theorem is proven using a set of results that we now establish. To this end, we first wish to prove the following bound:    
\begin{lemma}
For any $Q\in\mathcal{V}$ we have that
\begin{align}
    \left\|\mathfrak{F}Q-Q'\right\|\leq \gamma\left\|Q-Q'\right\|,
\end{align}
so that the operator $\mathfrak{F}$ is a contraction.
\end{lemma}
\begin{proof}
Recall, for any test function $\Lambda\in L_2$ , a projection operator $\cP$ acting $v$ is defined by the following 
\begin{align*}
\cP \Lambda:=\underset{\bar{\Lambda}\in\{\Phi r|r\in\mathbb{R}^p\}}{\arg\min}\left\|\bar{\Lambda}-\Lambda\right\|. 
\end{align*}
Now, we first note that in the proof of Lemma \ref{lemma:bellman_contraction}, we deduced that for any $\Lambda\in L_2$ we have that
\begin{align*}
    \left\|\mathcal{M}\Lambda-\left[ R(\cdot,a)+\gamma\underset{a\in\mathcal{A}}{\max}\;\mathcal{P}^a\Lambda'\right]\right\|\leq \gamma\left\|\Lambda-\Lambda'\right\|,
\end{align*}
(c.f. Lemma \ref{lemma:bellman_contraction}). 

Setting $\Lambda=Q$ and $\psi=\Theta$, it can be straightforwardly deduced that for any $Q,\hat{Q}\in L_2$:
    $\left\|\mathcal{M}Q-\hat{Q}\right\|\leq \gamma\left\|Q-\hat{Q}\right\|$. Hence, using the contraction property of $\mathcal{M}$, we readily deduce the following bound:
\begin{align}\max\left\{\left\|\mathcal{M}Q-\hat{Q}\right\|,\left\|\mathcal{M}Q-\mathcal{M}\hat{Q}\right\|\right\}\leq \gamma\left\|Q-\hat{Q}\right\|,
\label{m_bound_q_twice}
\end{align}
    
We now observe that $\mathfrak{F}$ is a contraction. Indeed, since for any $Q,Q'\in L_2$ we have that:
%
%
%
\begin{align*}
\left\|\mathfrak{F}Q-\mathfrak{F}Q'\right\|&=\left\|\Theta+\gamma P \max\{\mathcal{M}Q,Q\}-\left(\Theta+\gamma P \max\{\mathcal{M}Q',Q'\}\right)\right\|
\\&=\gamma \left\|P \max\{\mathcal{M}Q,Q\}-P \max\{\mathcal{M}Q',Q'\}\right\|
\\&\leq\gamma \left\| \max\{\mathcal{M}Q,Q\}- \max\{\mathcal{M}Q',Q'\}\right\|
\\&\leq\gamma \left\| \max\{\mathcal{M}Q-\mathcal{M}Q',Q-\mathcal{M}Q',\mathcal{M}Q-Q',Q-Q'\}\right\|
\\&\leq\gamma \max\{\left\|\mathcal{M}Q-\mathcal{M}Q'\right\|,\left\|Q-\mathcal{M}Q'\right\|,\left\|\mathcal{M}Q-Q'\right\|,\left\|Q-Q'\right\|\}
\\&=\gamma\left\|Q-Q'\right\|,
\end{align*}
using \eqref{m_bound_q_twice} and again using the non-expansiveness of $P$.
\end{proof}
We next show that the following two bounds hold:
\begin{lemma}\label{projection_F_contraction_lemma}
For any $Q\in\mathcal{V}$ we have that
\begin{itemize}
    \item[i)] 
$\qquad\qquad
    \left\|\cP \mathfrak{F}Q-\cP \mathfrak{F}\bar{Q}\right\|\leq \gamma\left\|Q-\bar{Q}\right\|$,
    \item[ii)]$\qquad\qquad\left\|\Phi r^\star - Q^\star\right\|\leq \frac{1}{\sqrt{1-\gamma^2}}\left\|\cP Q^\star - Q^\star\right\|$. 
\end{itemize}
\end{lemma}
\begin{proof}
The first result is straightforward since as $\cP$ is a projection it is non-expansive and hence:
\begin{align*}
    \left\|\cP \mathfrak{F}Q-\cP \mathfrak{F}\bar{Q}\right\|\leq \left\| \mathfrak{F}Q-\mathfrak{F}\bar{Q}\right\|\leq \gamma \left\|Q-\bar{Q}\right\|,
\end{align*}
using the contraction property of $\mathfrak{F}$. This proves i). For ii), we note that by the orthogonality property of projections we have that $\left\langle\Phi r^\star - \cP Q^\star,\Phi r^\star - \cP Q^\star\right\rangle$, hence we observe that:
\begin{align*}
    \left\|\Phi r^\star - Q^\star\right\|^2&=\left\|\Phi r^\star - \cP Q^\star\right\|^2+\left\|\Phi r^\star - \cP Q^\star\right\|^2
\\&=\left\|\cP \mathfrak{F}\Phi r^\star - \cP Q^\star\right\|^2+\left\|\Phi r^\star - \cP Q^\star\right\|^2
\\&\leq\left\|\mathfrak{F}\Phi r^\star -  Q^\star\right\|^2+\left\|\Phi r^\star - \cP Q^\star\right\|^2
\\&=\left\|\mathfrak{F}\Phi r^\star -  \mathfrak{F}Q^\star\right\|^2+\left\|\Phi r^\star - \cP Q^\star\right\|^2
\\&\leq\gamma^2\left\|\Phi r^\star -  Q^\star\right\|^2+\left\|\Phi r^\star - \cP Q^\star\right\|^2,
\end{align*}
after which we readily deduce the desired result.
\end{proof}

\begin{lemma}
Define  the operator $H$ by the following: $
  HQ(z)=  \begin{cases}
			\mathcal{M}Q(z), & \text{if $\mathcal{M}Q(z)>\Phi r^\star,$}\\
            Q(z), & \text{otherwise},
		 \end{cases}$
\\and $\tilde{\mathfrak{F}}$ by: $
    \tilde{\mathfrak{F}}Q:=\Theta +\gamma PHQ$.

For any $Q,\bar{Q}\in L_2$ we have that
\begin{align}
    \left\|\tilde{\mathfrak{F}}Q-\tilde{\mathfrak{F}}\bar{Q}\right\|\leq \gamma \left\|Q-\bar{Q}\right\|
\end{align}
and hence $\tilde{\mathfrak{F}}$ is a contraction mapping.
\end{lemma}
\begin{proof}
Using \eqref{m_bound_q_twice}, we now observe that
\begin{align*}
    \left\|\tilde{\mathfrak{F}}Q-\tilde{\mathfrak{F}}\bar{Q}\right\|&=\left\|\Theta+\gamma PHQ -\left(\Theta+\gamma PH\bar{Q}\right)\right\|
\\&\leq \gamma\left\|HQ - H\bar{Q}\right\|
\\&\leq \gamma\left\|\max\left\{\mathcal{M}Q-\mathcal{M}\bar{Q},Q-\bar{Q},\mathcal{M}Q-\bar{Q},\mathcal{M}\bar{Q}-Q\right\}\right\|
\\&\leq \gamma\max\left\{\left\|\mathcal{M}Q-\mathcal{M}\bar{Q}\right\|,\left\|Q-\bar{Q}\right\|,\left\|\mathcal{M}Q-\bar{Q}\right\|,\left\|\mathcal{M}\bar{Q}-Q\right\|\right\}
\\&\leq \gamma\max\left\{\gamma\left\|Q-\bar{Q}\right\|,\left\|Q-\bar{Q}\right\|,\left\|\mathcal{M}Q-\bar{Q}\right\|,\left\|\mathcal{M}\bar{Q}-Q\right\|\right\}
\\&=\gamma\left\|Q-\bar{Q}\right\|,
\end{align*}
again using the non-expansive property of $P$.
\end{proof}
\begin{lemma}
Define by $\tilde{Q}:=\Theta+\gamma Pv^{\boldsymbol{\tilde{\cP}}}$ where
\begin{align}
    v^{\boldsymbol{\tilde{\pi}}}(z):= \Theta(s_{\tau_k},a)+\gamma\underset{a\in\mathcal{A}}{\max}\;\sum_{s'\in\mathcal{S}}P(s';a,s_{\tau_k})\Phi r^\star(s',I(\tau_k)), \label{v_tilde_definition}
\end{align}
then $\tilde{Q}$ is a fixed point of $\tilde{\mathfrak{F}}\tilde{Q}$, that is $\tilde{\mathfrak{F}}\tilde{Q}=\tilde{Q}$. 
\end{lemma}
\begin{proof}
We begin by observing that
\begin{align*}
H\tilde{Q}(z)&=H\left(\Theta(z)+\gamma Pv^{\boldsymbol{\tilde{\pi}}}\right)    
\\&= \begin{cases}
			\mathcal{M}Q(z), & \text{if $\mathcal{M}Q(z)>\Phi r^\star,$}\\
            Q(z), & \text{otherwise},
		 \end{cases}
\\&= \begin{cases}
			\mathcal{M}Q(z), & \text{if $\mathcal{M}Q(z)>\Phi r^\star,$}\\
            \Theta(z)+\gamma Pv^{\boldsymbol{\tilde{\pi}}}, & \text{otherwise},
		 \end{cases}
\\&=v^{\boldsymbol{\tilde{\pi}}}(z).
\end{align*}
Hence,
\begin{align}
    \tilde{\mathfrak{F}}\tilde{Q}=\Theta+\gamma PH\tilde{Q}=\Theta+\gamma Pv^{\boldsymbol{\tilde{\pi}}}=\tilde{Q}. 
\end{align}
which proves the result.
\end{proof}
\begin{lemma}\label{value_difference_Q_difference}
The following bound holds:
\begin{align}
    \mathbb{E}\left[v^{\boldsymbol{\hat{\pi}}}(z_0)\right]-\mathbb{E}\left[v^{\boldsymbol{\tilde{\pi}}}(z_0)\right]\leq 2\left[(1-\gamma)\sqrt{(1-\gamma^2)}\right]^{-1}\left\|\cP Q^\star -Q^\star\right\|.
\label{F_tilde_fixed_point}\end{align}
\end{lemma}
\begin{proof}

By definitions of $v^{\boldsymbol{\hat{\pi}}}$ and $v^{\boldsymbol{\tilde{\pi}}}$ (c.f \eqref{v_tilde_definition}) and using Jensen's inequality and the stationarity property we have that,
\begin{align}\nonumber
    \mathbb{E}\left[v^{\boldsymbol{\hat{\pi}}}(z_0)\right]-\mathbb{E}\left[v^{\boldsymbol{\tilde{\pi}}}(z_0)\right]&=\mathbb{E}\left[Pv^{\boldsymbol{\hat{\pi}}}(z_0)\right]-\mathbb{E}\left[Pv^{\boldsymbol{\tilde{\pi}}}(z_0)\right]
    \\&\leq \left|\mathbb{E}\left[Pv^{\boldsymbol{\hat{\pi}}}(z_0)\right]-\mathbb{E}\left[Pv^{\boldsymbol{\tilde{\pi}}}(z_0)\right]\right|\nonumber
    \\&\leq \left\|Pv^{\boldsymbol{\hat{\pi}}}-Pv^{\boldsymbol{\tilde{\pi}}}\right\|. \label{v_approx_intermediate_bound_P}
\end{align}
Now recall that $\tilde{Q}:=\Theta+\gamma Pv^{\boldsymbol{\tilde{\pi}}}$ and $Q^\star:=\Theta+\gamma Pv^{\boldsymbol{\pi^\star}}$,  using these expressions in \eqref{v_approx_intermediate_bound_P} we find that 
\begin{align*}
    \mathbb{E}\left[v^{\boldsymbol{\hat{\pi}}}(z_0)\right]-\mathbb{E}\left[v^{\boldsymbol{\tilde{\pi}}}(z_0)\right]&\leq \frac{1}{\gamma}\left\|\tilde{Q}-Q^\star\right\|. \label{v_approx_q_approx_bound}
\end{align*}
Moreover, by the triangle inequality and using the fact that $\mathfrak{F}(\Phi r^\star)=\tilde{\mathfrak{F}}(\Phi r^\star)$ and that $\mathfrak{F}Q^\star=Q^\star$ and $\mathfrak{F}\tilde{Q}=\tilde{Q}$ (c.f. \eqref{F_tilde_fixed_point}) we have that
\begin{align*}
\left\|\tilde{Q}-Q^\star\right\|&\leq \left\|\tilde{Q}-\mathfrak{F}(\Phi r^\star)\right\|+\left\|Q^\star-\tilde{\mathfrak{F}}(\Phi r^\star)\right\|    
\\&\leq \gamma\left\|\tilde{Q}-\Phi r^\star\right\|+\gamma\left\|Q^\star-\Phi r^\star\right\| 
\\&\leq 2\gamma\left\|\tilde{Q}-\Phi r^\star\right\|+\gamma\left\|Q^\star-\tilde{Q}\right\|, 
\end{align*}
which gives the following bound:
\begin{align*}
\left\|\tilde{Q}-Q^\star\right\|&\leq 2\left(1-\gamma\right)^{-1}\left\|\tilde{Q}-\Phi r^\star\right\|, 
\end{align*}
from which, using Lemma \ref{projection_F_contraction_lemma}, we deduce that $
    \left\|\tilde{Q}-Q^\star\right\|\leq 2\left[(1-\gamma)\sqrt{(1-\gamma^2)}\right]^{-1}\left\|\tilde{Q}-\Phi r^\star\right\|$,
after which by \eqref{v_approx_q_approx_bound}, we finally obtain
\begin{align*}
        \mathbb{E}\left[v^{\boldsymbol{\hat{\pi}}}(z_0)\right]-\mathbb{E}\left[v^{\boldsymbol{\tilde{\pi}}}(z_0)\right]\leq  2\left[(1-\gamma)\sqrt{(1-\gamma^2)}\right]^{-1}\left\|\tilde{Q}-\Phi r^\star\right\|,
\end{align*}
as required.
\end{proof}

Let us rewrite the update in the following way:
\begin{align*}
    r_{t+1}=r_t+\gamma_t\Xi(w_t,r_t),
\end{align*}
where the function $\Xi:\mathbb{R}^{2d}\times \mathbb{R}^p\to\mathbb{R}^p$ is given by:
\begin{align*}
\Xi(w,r):=\phi(z)\left(\Theta(z)+\gamma\max\left\{(\Phi r) (z'),\mathcal{M}(\Phi r) (z')\right\}-(\Phi r)(z)\right),
\end{align*}
for any $w\equiv (z,z')\in\left(\mathbb{N}\times\mathcal{S}\right)^2$ where $z=(t,s)\in\mathbb{N}\times\mathcal{S}$ and $z'=(t,s')\in\mathbb{N}\times\mathcal{S}$  and for any $r\in\mathbb{R}^p$. Let us also define the function $\boldsymbol{\Xi}:\mathbb{R}^p\to\mathbb{R}^p$ by the following:
\begin{align*}
    \boldsymbol{\Xi}(r):=\mathbb{E}_{w_0\sim (\mathbb{P},\mathbb{P})}\left[\Xi(w_0,r)\right]; w_0:=(z_0,z_1).
\end{align*}
\begin{lemma}\label{iteratation_property_lemma}
The following statements hold for all $z\in \{0,1\}\times \mathcal{S}$:
\begin{itemize}
    \item[i)] $
(r-r^\star)\boldsymbol{\Xi}_k(r)<0,\qquad \forall r\neq r^\star,    
$
\item[ii)] $
\boldsymbol{\Xi}_k(r^\star)=0$.
\end{itemize}
\end{lemma}
\begin{proof}
To prove the statement, we first note that each component of $\boldsymbol{\Xi}_k(r)$ admits a representation as an inner product, indeed: 
\begin{align*}
\boldsymbol{\Xi}_k(r)&=\mathbb{E}\left[\phi_k(z_0)(\Theta(z_0)+\gamma\max\left\{\Phi r(z_1),\mathcal{M}\Phi(z_1)\right\}-(\Phi r)(z_0)\right] 
\\&=\mathbb{E}\left[\phi_k(z_0)(\Theta(z_0)+\gamma\mathbb{E}\left[\max\left\{\Phi r(z_1),\mathcal{M}\Phi(z_1)\right\}|z_0\right]-(\Phi r)(z_0)\right]
\\&=\mathbb{E}\left[\phi_k(z_0)(\Theta(z_0)+\gamma P\max\left\{\left(\Phi r,\mathcal{M}\Phi\right)\right\}(z_0)-(\Phi r)(z_0)\right]
\\&=\left\langle\phi_k,\mathfrak{F}\Phi r-\Phi r\right\rangle,
\end{align*}
using the iterated law of expectations and the definitions of $P$ and $\mathfrak{F}$.

We now are in position to prove i). Indeed, we now observe the following:
\begin{align*}
\left(r-r^\star\right)\boldsymbol{\Xi}_k(r)&=\sum_{l=1}\left(r(l)-r^\star(l)\right)\left\langle\phi_l,\mathfrak{F}\Phi r -\Phi r\right\rangle
\\&=\left\langle\Phi r -\Phi r^\star, \mathfrak{F}\Phi r -\Phi r\right\rangle
\\&=\left\langle\Phi r -\Phi r^\star, (\boldsymbol{1}-\cP)\mathfrak{F}\Phi r+\cP \mathfrak{F}\Phi r -\Phi r\right\rangle
\\&=\left\langle\Phi r -\Phi r^\star, \cP \mathfrak{F}\Phi r -\Phi r\right\rangle,
\end{align*}
where in the last step we used the orthogonality of $(\boldsymbol{1}-\cP)$. We now recall that $\cP \mathfrak{F}\Phi r^\star=\Phi r^\star$ since $\Phi r^\star$ is a fixed point of $\cP \mathfrak{F}$. Additionally, using Lemma \ref{projection_F_contraction_lemma} we observe that $\|\cP \mathfrak{F}\Phi r -\Phi r^\star\| \leq \gamma \|\Phi r -\Phi r^\star\|$. With this we now find that
\begin{align*}
&\left\langle\Phi r -\Phi r^\star, \cP \mathfrak{F}\Phi r -\Phi r\right\rangle    
\\&=\left\langle\Phi r -\Phi r^\star, (\cP \mathfrak{F}\Phi r -\Phi r^\star)+ \Phi r^\star -\Phi r\right\rangle
\\&\leq\left\|\Phi r -\Phi r^\star\right\|\left\|\cP \mathfrak{F}\Phi r -\Phi r^\star\right\|- \left\|\Phi r^\star -\Phi r\right\|^2
\\&\leq(\gamma -1)\left\|\Phi r^\star -\Phi r\right\|^2,
\end{align*}
which is negative since $\gamma<1$ which completes the proof of part i).

The proof of part ii) is straightforward since we readily observe that
\begin{align*}
    \boldsymbol{\Xi}_k(r^\star)= \left\langle\phi_l, \mathfrak{F}\Phi r^\star-\Phi r\right\rangle= \left\langle\phi_l, \cP \mathfrak{F}\Phi r^\star-\Phi r\right\rangle=0,
\end{align*}
as required and from which we deduce the result.
\end{proof}
To prove the theorem, we make use of a special case of the following result:

\begin{theorem}[Th. 17, p. 239 in \cite{benveniste2012adaptive}] \label{theorem:stoch.approx.}
Consider a stochastic process $r_t:\mathbb{R}\times\{\infty\}\times\Omega\to\mathbb{R}^k$ which takes an initial value $r_0$ and evolves according to the following:
\begin{align}
    r_{t+1}=r_t+\alpha \Xi(s_t,r_t),
\end{align}
for some function $s:\mathbb{R}^{2d}\times\mathbb{R}^k\to\mathbb{R}^k$ and where the following statements hold:
\begin{enumerate}
    \item $\{s_t|t=0,1,\ldots\}$ is a stationary, ergodic Markov process taking values in $\mathbb{R}^{2d}$
    \item For any positive scalar $q$, there exists a scalar $\mu_q$ such that $\mathbb{E}\left[1+\|s_t\|^q|s\equiv s_0\right]\leq \mu_q\left(1+\|s\|^q\right)$
    \item The step size sequence satisfies the Robbins-Monro conditions, that is $\sum_{t=0}^\infty\alpha_t=\infty$ and $\sum_{t=0}^\infty\alpha^2_t<\infty$
    \item There exists scalars $c$ and $q$ such that $    \|\Xi(w,r)\|
        \leq c\left(1+\|w\|^q\right)(1+\|r\|)$
    \item There exists scalars $c$ and $q$ such that $
        \sum_{t=0}^\infty\left\|\mathbb{E}\left[\Xi(w_t,r)|z_0\equiv z\right]-\mathbb{E}\left[\Xi(w_0,r)\right]\right\|
        \leq c\left(1+\|w\|^q\right)(1+\|r\|)$
    \item There exists a scalar $c>0$ such that $
        \left\|\mathbb{E}[\Xi(w_0,r)]-\mathbb{E}[\Xi(w_0,\bar{r})]\right\|\leq c\|r-\bar{r}\| $
    \item There exists scalars $c>0$ and $q>0$ such that $
        \sum_{t=0}^\infty\left\|\mathbb{E}\left[\Xi(w_t,r)|w_0\equiv w\right]-\mathbb{E}\left[\Xi(w_0,\bar{r})\right]\right\|
        \leq c\|r-\bar{r}\|\left(1+\|w\|^q\right) $
    \item There exists some $r^\star\in\mathbb{R}^k$ such that $\boldsymbol{\Xi}(r)(r-r^\star)<0$ for all $r \neq r^\star$ and $\bar{s}(r^\star)=0$. 
\end{enumerate}
Then $r_t$ converges to $r^\star$ almost surely.
\end{theorem}

In order to apply the Theorem \ref{theorem:stoch.approx.}, we show that conditions 1 - 7 are satisfied.

\begin{proof}
Conditions 1-2 are true by assumption while condition 3 can be made true by choice of the learning rates. Therefore it remains to verify conditions 4-7 are met.   

To prove 4, we observe that
\begin{align*}
\left\|\Xi(w,r)\right\|
&=\left\|\phi(z)\left(\Theta(z)+\gamma\max\left\{(\Phi r) (z'),\mathcal{M}\Phi (z')\right\}-(\Phi r)(z)\right)\right\|
\\&\leq\left\|\phi(z)\right\|\left\|\Theta(z)+\gamma\left(\left\|\phi(z')\right\|\|r\|+\mathcal{M}\Phi (z')\right)\right\|+\left\|\phi(z)\right\|\|r\|
\\&\leq\left\|\phi(z)\right\|\left(\|\Theta(z)\|+\gamma\|\mathcal{M}\Phi (z')\|\right)+\left\|\phi(z)\right\|\left(\gamma\left\|\phi(z')\right\|+\left\|\phi(z)\right\|\right)\|r\|.
\end{align*}
Now using the definition of $\mathcal{M}$, we readily observe that $\|\mathcal{M}\Phi (z')\|\leq \| \Theta\|+\gamma\|\mathcal{P}^\pi_{s's_t}\Phi\|\leq \| \Theta\|+\gamma\|\Phi\|$ using the non-expansiveness of $P$.

Hence, we lastly deduce that
\begin{align*}
\left\|\Xi(w,r)\right\|
&\leq\left\|\phi(z)\right\|\left(\|\Theta(z)\|+\gamma\|\mathcal{M}\Phi (z')\|\right)+\left\|\phi(z)\right\|\left(\gamma\left\|\phi(z')\right\|+\left\|\phi(z)\right\|\right)\|r\|
\\&\leq\left\|\phi(z)\right\|\left(\|\Theta(z)\|+\gamma\| \Theta\|+\gamma\|\psi\|\right)+\left\|\phi(z)\right\|\left(\gamma\left\|\phi(z')\right\|+\left\|\phi(z)\right\|\right)\|r\|,
\end{align*}
we then easily deduce the result using the boundedness of $\phi,\Theta$ and $\psi$.

Now we observe the following Lipschitz condition on $\Xi$:
\begin{align*}
&\left\|\Xi(w,r)-\Xi(w,\bar{r})\right\|
\\&=\left\|\phi(z)\left(\gamma\max\left\{(\Phi r)(z'),\mathcal{M}\Phi(z')\right\}-\gamma\max\left\{(\Phi \bar{r})(z'),\mathcal{M}\Phi(z')\right\}\right)-\left((\Phi r)(z)-\Phi\bar{r}(z)\right)\right\|
\\&\leq\gamma\left\|\phi(z)\right\|\left\|\max\left\{\phi'(z') r,\mathcal{M}\Phi'(z')\right\}-\max\left\{(\phi'(z') \bar{r}),\mathcal{M}\Phi'(z')\right\}\right\|+\left\|\phi(z)\right\|\left\|\phi'(z) r-\phi(z)\bar{r}\right\|
\\&\leq\gamma\left\|\phi(z)\right\|\left\|\phi'(z') r-\phi'(z') \bar{r}\right\|+\left\|\phi(z)\right\|\left\|\phi'(z) r-\phi'(z)\bar{r}\right\|
\\&\leq \left\|\phi(z)\right\|\left(\left\|\phi(z)\right\|+ \gamma\left\|\phi(z)\right\|\left\|\phi'(z') -\phi'(z') \right\|\right)\left\|r-\bar{r}\right\|
\\&\leq c\left\|r-\bar{r}\right\|,
\end{align*}
using Cauchy-Schwarz inequality and  that for any scalars $a,b,c$ we have that \\$
    \left|\max\{a,b\}-\max\{b,c\}\right|\leq \left|a-c\right|$.
    
Using Assumptions 3 and 4, we therefore deduce that
\begin{align}
\sum_{t=0}^\infty\left\|\mathbb{E}\left[\Xi(w,r)-\Xi(w,\bar{r})|w_0=w\right]-\mathbb{E}\left[\Xi(w_0,r)-\Xi(w_0,\bar{r})\right\|\right]\leq c\left\|r-\bar{r}\right\|(1+\left\|w\right\|^l).
\end{align}

Part 2 is assured by Lemma \ref{projection_F_contraction_lemma} while Part 4 is assured by Lemma \ref{value_difference_Q_difference} and lastly Part 8 is assured by Lemma \ref{iteratation_property_lemma}.
This result completes the proof of Theorem \ref{theorem:existence}. 
\end{proof}

\section*{Proof of Proposition \ref{NE_improve_prop}}
\begin{proof}[Proof of Prop. \ref{NE_improve_prop}]
We split the proof into two parts: 

i) We first prove that $v^{\tilde{\pi}}(s)\geq v^{\pi}(s),\;\forall s \in\mathcal{S}$ where we used $\tilde{\pi}$ to denote the {\fontfamily{cmss}\selectfont Controller}'s policy induced under the influence of the {\fontfamily{cmss}\selectfont Shaper}.

ii) Second, we prove that there exists a finite integer $M$ such that  $v^{\tilde{\pi}_m}(s)\geq v^{\pi_m}(s)$ for any $m\geq M$.

The proof of part (i) is achieved by proof by contradiction.
Denote by $v^{{\pi},\pi^2\equiv {0}}$ the value function for the {\fontfamily{cmss}\selectfont Controller} for the system \textit{without the {\fontfamily{cmss}\selectfont Shaper}} and its shaping reward function. 
Indeed, 
let $({\hat{\pi}},\hat{\pi}^2)$ be the policy profile induced by the MPE policy profile and assume that the shaping reward $F$ leads to a decrease in payoff for the {\fontfamily{cmss}\selectfont Controller}. Then by construction 
$    v^{{\pi},\pi^2}(s)< v^{{\pi},\pi^2\equiv {0}}(s)
$
which is a contradiction since $({\hat{\pi}},\hat{\pi}^2)$ is an MPE profile. To arrive at the required result, we invoke Prop. \ref{invariance_prop}, which proves that the inclusion of the {\fontfamily{cmss}\selectfont Shaper} does not affect the total expected return, we can therefore conclude that the inequality holds for the extrinsic value functions and that the {\fontfamily{cmss}\selectfont Shaper} induces a {\fontfamily{cmss}\selectfont Controller} policy that leads to a (weakly) higher expected return from the \textit{environment}. Hence, we have succeeded in showing that $v^{\tilde{\pi}}(s)\geq v^{\pi}(s),\;\forall s \in\mathcal{S}$.

We now prove part (ii). 

By part (i) we have that $v^{\tilde{\pi}}(s)=\underset{m\to\infty}{\lim}v^{\tilde{\pi}_m}(s)\geq v^{\pi}(s)= \underset{m\to\infty}{\lim}v^{\pi_m}(s)$. Since $v^{\pi}(s)$ is maximal in the sequence $v^{\pi_1}(s),v^{\pi_2}(s),\ldots,v^{\pi}(s)$ we can deduce that $v^{\pi}(s)\geq v^{\pi_n}(s)$ for any $n\leq \infty$. Hence for any $n$ there exists a $c\geq 0$ such that $\underset{m\to\infty}{\lim}v^{\tilde{\pi}_m}(s)= v^{\pi_n}(s)+c$.
Now by construction $v^{\tilde{\pi}_m}(s)\to v^{\tilde{\pi}}(s)$ as $m\to\infty$, therefore the sequence $\tilde{v}^{\pi_1},\tilde{v}^{\pi_2},\ldots,$ forms a Cauchy sequence. Therefore, there exists an $M$ such that for any $\epsilon >0$, $v^{\tilde{\pi}_n}(s)- (v^{\pi_n}(s)+c)< \epsilon$ $\forall n\geq M$. Since $\epsilon$ is arbitrary we can conclude that $v^{\tilde{\pi}_n}(s)- (v^{\pi_n}(s)+c)= 0$ $\forall n\geq M$. Since $c\geq 0$, we immediately deduce that $v^{\tilde{\pi}_n}(s)\geq v^{\pi_n}(s), \forall n\geq M$  which is the required result.    
\end{proof}
This result completes the proof of Theorem \ref{main_theorem}.
\end{proof}

Having constructed a procedure to find the optimal Controller policy, our next result characterises {\fontfamily{cmss}\selectfont Shaper} policy $\mathfrak{g}_2$ and the optimal times to activate $F$. 
\begin{proposition}\label{prop:switching_times}
The policy $\mathfrak{g}_2$ is given by the following expression: $\mathfrak{g}_2(s_t)=H(\mathcal{M}^{\pi,\pi^2}V^{\pi,\pi^2}- V^{\pi,\pi^2})(s_t,I_t),\;\;\forall (s_t,I_t)\in\mathcal{S}\times\{0,1\}$, where $V$ is the solution in Theorem \ref{theorem:existence} and $H$ is the Heaviside function, moreover {\fontfamily{cmss}\selectfont Shaper}'s switching times 
are $\tau_k=\inf\{\tau>\tau_{k-1}|\mathcal{M}^{\pi,\pi^2}V^{\pi,\pi^2}= V^{\pi,\pi^2}\}$.

\end{proposition}

Hence, Prop. \ref{prop:switching_times} also characterises the (categorical) distribution $\mathfrak{g}_2$. Moreover, given the function $V$, the times $\{\tau_k\}$ can be determined by evaluating if $\mathcal{M}V=V$ holds.

\section*{Proof of Proposition \ref{prop:switching_times}}
\begin{proof}[Proof of Prop. \ref{prop:switching_times}]
The proof is given by establishing a contradiction. Therefore suppose that $\mathcal{M}^{\pi,\pi^2}\psi(s_{\tau_k},I(\tau_k))\leq \psi(s_{\tau_k},I(\tau_k))$ and suppose that the intervention time $\tau'_1>\tau_1$ is an optimal intervention time. Construct the Player 2 $\pi'^2\in\Pi^2$ and $\tilde{\pi}^2$ policy switching times by $(\tau'_0,\tau'_1,\ldots,)$ and $\pi'^2\in\Pi^2$ policy by $(\tau'_0,\tau_1,\ldots)$ respectively.  Define by $l=\inf\{t>0;\mathcal{M}^{\pi,\pi^2}\psi(s_{t},I_0)= \psi(s_{t},I_0)\}$ and $m=\sup\{t;t<\tau'_1\}$.
By construction we have that
\begin{align*}
& \quad v^{\pi^1,\pi'^2}_2(s,I_0)
\\&\begin{aligned}=\mathbb{E}\Bigg[R(s_{0},a_{0})+\mathbb{E}\Bigg[\ldots+\gamma^{l-1}\mathbb{E}\Bigg[R(s_{\tau_1-1},a_{\tau_1-1})+\ldots+\gamma^{m-l-1}\mathbb{E}\Bigg[ R(s_{\tau'_1-1},a_{\tau'_1-1})&
\\+\gamma\mathcal{M}^{\pi^1,\pi'^2}v^{\pi^1,\pi'^2}_2(s',I(\tau'_{1}))\Bigg]\Bigg]\Bigg]\Bigg]&
\end{aligned}
\\&<\mathbb{E}\left[R(s_{0},a_{0})+\mathbb{E}\left[\ldots+\gamma^{l-1}\mathbb{E}\left[ R(s_{\tau_1-1},a_{\tau_1-1})+\gamma\mathcal{M}^{\pi^1,\tilde{\pi}^2}v^{\pi^1,\pi'^2}_2(s_{\tau_1},I(\tau_1))\right]\right]\right]
\end{align*}
We now use the following observation 
\begin{align*}
&\mathbb{E}\left[ R(s_{\tau_1-1},a_{\tau_1-1})+\gamma\mathcal{M}^{\pi^1,\tilde{\pi}^2}v^{\pi^1,\pi'^2}_2(s_{\tau_1},I(\tau_1))\right]
\\&\leq \max\left\{\mathcal{M}^{\pi^1,\tilde{\pi}^2}v^{\pi^1,\pi'^2}_2(s_{\tau_1},I(\tau_1)),\underset{a_{\tau_1}\in\mathcal{A}}{\max}\;\left[ R(s_{\tau_{k}},a_{\tau_{k}})+\gamma\sum_{s'\in\mathcal{S}}P(s';a_{\tau_1},s_{\tau_1})v^{\pi^1,\pi^2}_2(s',I(\tau_1))\right]\right\}.
\end{align*}
Using this we deduce that
\begin{align*}
&v^{\pi^1,\pi'^2}_2(s,I_0)\leq\mathbb{E}\Bigg[R(s_{0},a_{0})+\mathbb{E}\Bigg[\ldots
\\&\begin{aligned}+\gamma^{l-1}\mathbb{E}\Bigg[ R(s_{\tau_1-1},a_{\tau_1-1})+&\gamma\max\Bigg\{\mathcal{M}^{\pi^1,\tilde{\pi}^2}v^{\pi^1,\pi'^2}_2(s_{\tau_1},I(\tau_1)),
\\&\underset{a_{\tau_1}\in\mathcal{A}}{\max}\;\Bigg[ R(s_{\tau_{k}},a_{\tau_{k}})+\gamma\sum_{s'\in\mathcal{S}}P(s';a_{\tau_1},s_{\tau_1})v^{\pi^1,\pi^2}_2(s',I(\tau_1))\Bigg]\Bigg\}\Bigg]\Bigg]\Bigg]
\end{aligned}
\\&=\mathbb{E}\left[R(s_{0},a_{0})+\mathbb{E}\left[\ldots+\gamma^{l-1}\mathbb{E}\left[ R(s_{\tau_1-1},a_{\tau_1-1})+\gamma\left[T v^{\pi^1,\tilde{\pi}^2}_2\right](s_{\tau_1},I(\tau_1))\right]\right]\right]=v^{\pi^1,\tilde{\pi}^2}_2(s,I_0)),
\end{align*}
where the first inequality is true by assumption on $\mathcal{M}$. This is a contradiction since $\pi'^2$ is an optimal policy for Player 2. Using analogous reasoning, we deduce the same result for $\tau'_k<\tau_k$ after which deduce the result. Moreover, by invoking the same reasoning, we can conclude that it must be the case that $(\tau_0,\tau_1,\ldots,\tau_{k-1},\tau_k,\tau_{k+1},\ldots,)$ are the optimal switching times.

\end{proof}

\end{document}